\documentclass[10pt,letterpaper]{article}
\usepackage[margin=1.01in]{geometry} 
\usepackage{graphicx}
\usepackage{subcaption}
\usepackage{amsmath}
\usepackage{amsthm}
\usepackage{natbib}
\usepackage{hyperref}
\usepackage{amsfonts}
\usepackage{amssymb}
\usepackage{algorithm}
\usepackage{algorithmic}
\usepackage{mathtools}
\usepackage{xcolor}
\usepackage{booktabs}   
\usepackage{caption}    
\usepackage{float}      
\usepackage{pifont}
\usepackage{authblk}

\title{CausalGDP: Causality-Guided Diffusion Policies for Reinforcement Learning}
\author[1]{Xiaofeng Xiao}
\author[2]{Xiao Hu}
\author[2]{Yang Ye}
\author[1]{Xubo Yue\thanks{Corresponding Author: \texttt{x.yue@northeastern.edu}}}
\affil[1]{Department of Mechanical \& Industrial Engineering, Northeastern University, Boston, MA, USA}
\affil[2]{Department of Civil and Environmental Engineering, Northeastern University, Boston, MA, USA}
 
\date{}

\begin{document}
\maketitle
\begin{abstract}
Reinforcement learning (RL) has achieved remarkable success in a wide range of sequential decision-making problems. Recent diffusion-based policies further improve RL by modeling complex, high-dimensional action distributions. However, existing diffusion policies primarily rely on statistical associations and fail to explicitly account for causal relationships among states, actions, and rewards, limiting their ability to identify which action components truly cause high returns. In this paper, we propose Causality-guided Diffusion Policy (CausalGDP), a unified framework that integrates causal reasoning into diffusion-based RL. CausalGDP first learns a base diffusion policy and an initial causal dynamical model from offline data, capturing causal dependencies among states, actions, and rewards. During real-time interaction, the causal information is continuously updated and incorporated as a guidance signal to steer the diffusion process toward actions that causally influence future states and rewards. By explicitly considering causality beyond association, CausalGDP focuses policy optimization on action components that genuinely drive performance improvements. Experimental results demonstrate that CausalGDP consistently achieves competitive or superior performance over state-of-the-art diffusion-based and offline RL methods, especially in complex, high-dimensional control tasks.

\end{abstract}


\section{Introduction}
Reinforcement learning (RL) aims to learn a policy that can complete complex tasks. RL can be broadly categorized into offline and online frameworks. Offline RL learns from a static, pre-collected dataset without interacting with the environment \citep{levine2020offline, fujimoto2021minimalist, kostrikov2021offline}, making it suitable for settings where data collection is costly or risky, such as optimizing treatments from historical patient records or autonomous driving. Online RL \citep{wang2021online, mhammedi2024power}, in contrast, interacts directly with the environment, updating the policy based on feedback. This enables adaptation to dynamic or novel scenarios, such as real-time recommendation systems, robotic control in factory automation, and game-playing AI improving via self-play. 


As environments and real-world scenarios become increasingly complex, simple RL policy is often insufficient. Specifically, it is challenging for an RL framework to learn a general policy distribution and execute actions effectively when both the action space and observations are high-dimensional (e.g., a robot may have over 20 action dimensions and more than 300 observation dimensions). Therefore, some recent works employ diffusion models to represent the policy in RL frameworks \citep{wang2022diffusion, kang2023efficient, ren2024diffusion}. A diffusion model \citep{croitoru2023diffusion, yang2023diffusion, cao2024survey} is a generative model that can produce samples by denoising corrupted data. This capability allows it to effectively capture and model high-dimensional and complex policy distributions in RL. However, due to the complexity of RL policies involving high-dimensional states and actions, it is challenging for a standard diffusion model to generate a reasonable policy. Recently, guidance mechanisms \citep{ni2023metadiffuser, mao2024diffusion} have been introduced into diffusion models. For instance, the optimal policy, often represented via the $Q$-value in many studies, can be used to guide the diffusion policy. Specifically, the action generation is influenced by the gradient of the $Q$-value function, steering sampled actions toward directions with higher expected return. Moreover, the guidance can also be based on the divergence between high-reward actions and the actions sampled from the diffusion policy \citep{KumarZTL20, jackson2024policy}, where smaller divergence encourages the diffusion policy to better approximate the optimal actions.

However, a key factor that has not been fully addressed in diffusion policies is \textbf{Causality}. In reinforcement learning, the environment is modeled as a Markov Decision Process (MDP), where actions $a_t$ and states $s_t$ evolve by time $t$ according to the transition dynamics $s_{t+1} \sim P\left(s_{t+1} \mid s_t, a_t\right).$ This probabilistic framework captures the statistical associations among states, actions, and rewards $R\left(s_t, a_t\right)$. Yet, an important question remains: \textbf{if we control for all other confounders, how does changing a single component of the action vector, such as an element in $a_t$, actually cause changes in future states and rewards?} For a MDP, confounders are hidden or entangled dependencies inside the state, action, reward, or transition function. In practice, different components of the actions and states often change together, making it impossible to isolate the causal effect of a single component to next time states and rewards. In this setting, association is insufficient, because it cannot disentangle the effect of an individual action component from other confounders in the MDP. Identifying the effect of changing only $a_i$ while holding other factors fixed requires a causal intervention by manipulating the structural dependencies within the MDP.
By leveraging the causal dependencies intrinsically encoded in the transition and reward mechanisms, we can identify which specific action components have true causal influence on desirable future states thereby enabling diffusion policies to focus on the actions that genuinely cause high rewards. 

In this paper, we propose a \textbf{Causality-guided Diffusion Policy (CausalGDP)} framework that leverages causal information within the diffusion policy to generate actions. \textbf{CausalGDP} consists of two major stages: offline training and real-time learning: (\textbf{1) Offline Stage:}
The offline dataset is first used to train a base diffusion policy and a noise prediction network for denoising. In addition, a fundamental dynamical model is learned to capture the causal relationships among states, actions, and rewards. This provides an initial causal structure that will guide the policy;
\textbf{(2) Real-time Stage:} During interaction with the training environment, causal information in the learned dynamical model is continually updated using the incoming data stream. The updated causal structure is then incorporated into the diffusion policy as a causal guidance term, allowing the policy to sample actions that are causally aligned with the future states and higher rewards. In this process, the RL framework avoids actions that causally lead to lower rewards during training, while favoring actions that yield higher rewards. As a result, the training process is accelerated.
Our \textbf{contributions} of this work are summarized as follows:
\begin{itemize}
    \item We introduce a novel causal guidance mechanism for diffusion policies, which leverages real-time causal information rather than mere statistical associations;
    \item The proposed guidance is model-agnostic and can be seamlessly integrated into various diffusion-based policy architectures. 
    \item To the best of our knowledge, this is the first work that incorporates explicit causal reasoning to guide diffusion policies within an RL framework.
\end{itemize}

The remainder of this paper is organized as follows. In Section \ref{sec:Ltv} we review recent related work and highlight the key differences between prior approaches and our method, emphasizing the specific improvements we proposed. In Section \ref{sec:pre_s} we present the preliminary setup, including the general diffusion model and guidance mechanisms used for diffusion policies. Section \ref{sec:method} describes the proposed causality-guided diffusion policy in detail. In Section \ref{sec:theory} we provide theoretical analysis of our method. Finally, Section \ref{sec:exper} reports the experimental results.

\section{Literature Review}\label{sec:Ltv}
\paragraph{Causal Reinforcement Learning} A growing body of work has emphasized the importance of incorporating causality into the RL framework. \cite{bareinboim2021introduction} provided comprehensive overviews of how causal reasoning can be embedded within RL algorithms. In particular, \cite{bareinboim2021introduction} highlighted that Markov Decision Processes may contain latent confounders, especially when learning from offline observational datasets. This study provides the theoretical foundation for why performing interventions during online RL is not only feasible but also essential for uncovering causal structure. Moreover, \cite{grimbly2021causal,wang2021provably} suggested that an MDP can be viewed as a special case of a Structural Causal Model (SCM) over states, actions, and rewards, and online interaction is able to implemented explicit interventions, then exposing latent structure and restores the validity of causal reasoning in RL. \cite{hu2022causality} demonstrated that causal discovery techniques can be directly applied to factorize the causal dependencies within an MDP in a manner consistent with the environment’s underlying SCM. \cite{deng2023causal,zeng2024survey} further summarized that the RL framework is compatible with general causal discovery techniques once basic assumptions—such as the Markov condition and the rules of do-calculus—are satisfied through appropriate interventions. The prior works provided strong evidence that incorporating causality into reinforcement learning is both theoretically feasible and practically beneficial, which supports the motivation of our approach.

\paragraph{Diffusion Policy} Diffusion models have recently been adopted for policy generation in reinforcement learning, particularly in environments with complex dynamics and high-dimensional state–action spaces. \cite{janner2022planning} treated an entire trajectory—i.e., a sequence of states and actions—as a single diffusion sample. The diffusion guidance in this work relies solely on offline rewards, and the long-horizon trajectory denoising process may accumulate compounding errors. To alleviate these issues, \cite{wang2022diffusion} applied diffusion models at the action level rather than the full trajectory. Their method used $Q$-value–based guidance to sample actions during $Q$-learning, thereby avoiding the long-memory trajectory diffusion problem and reducing deviation from the true value-based policy update. \cite{kang2023efficient} extend the idea of $Q$-value guidance to a score-based diffusion model, allowing the diffusion policy to be sampled more efficiently from an ODE solver. This method could generate actions directly without running a full denoising trajectory. \cite{chen2024diffusion} proposed a new training objective termed a trust-region diffusion loss, which enables sampling actions from the diffusion policy while ensuring policy updates remain within a stable trust region. The guidance mechanism still relies on reward signals and the learned $Q$-values.  \cite{mao2024diffusion} decomposed the optimal policy’s score function into the behavior policy’s score (which is learned via a diffusion model) plus an additional correction term. The diffusion model then generated multiple action candidates under this guided score, and the correction term subsequently selected the optimal action among them. The above methods primarily rely on rewards or $Q$-values to guide the diffusion policy without accounting for the underlying causal structure of the MDP.

Few studies have explored incorporating causal structure into diffusion policies. Recently, \cite{chi2025diffusion} introduced action masks within a Transformer-based attention framework, where the masks restrict the action space by enforcing temporal causal dependencies between $a_t$ and $a_{t+1}$. However, this method only models causality across time of action and does not address causal dependencies among all components of the MDP. Moreover, this approach is tightly coupled with a Transformer architecture and does not provide general causal guidance for diffusion policies. Our approach incorporates causal information as a guidance mechanism for the entire diffusion policy. Especially, it provides a unified and scalable framework that can be applied across different reinforcement learning settings and diffusion models. 

We provide Table~\ref{tab:comparison} to present a comparison of our proposed \textbf{CausalGDP} with other recent diffusion-policy methods.

\begin{table}[H]
\centering
\captionsetup{font=footnotesize}
\resizebox{\textwidth}{!}{
\begin{tabular}{@{}lcccccc@{}}
\toprule
\textbf{Method} & \textbf{ Real-time Causality} & \textbf{Diffusion Policy} & \textbf{Model-agnostic Guidance }  &\textbf{Associations}  \\ \midrule
\textbf{CausalGDP} & $\checkmark$ & $\checkmark$& $\checkmark$  & $\checkmark$\\
\textbf{Offline-Diffusion}\citep{wang2022diffusion} & \ding{55}& $\checkmark$ & \ding{55}  & $\checkmark$   \\
\textbf{Diffusion-TrustRegion}\citep{chen2024diffusion} & \ding{55} &$\checkmark$ & \ding{55}& $\checkmark$\\
\textbf{Efficient-Diffusion}\citep{kang2023efficient} & \ding{55} & $\checkmark$ & $\checkmark$ & $\checkmark$ \\
\textbf{Attention-Diffusion}\citep{chi2025diffusion} & \ding{55} & $\checkmark$ & \ding{55} &  $\checkmark$ \\
\bottomrule
\end{tabular}
}
\caption{
Comparison of representative diffusion-policy methods. 
We evaluate whether each method: (1) incorporates dynamical causal information among components and allows interventions (\textbf{Causality}); 
(2) is built upon diffusion models for control (\textbf{Diffusion Policy}); 
(3) provides scalable and framework-independent guidance (\textbf{Model-agnostic Guidance}); 
 and 
(4) relies on statistical associations between variables (\textbf{Associations}).
}\label{tab:comparison}
\end{table}

\section{Preliminary Setting}\label{sec:pre_s}
\subsection{Diffusion Model}\label{sec:diff}
Diffusion model is a generative framework consisted of froward and reverse processes \citep{song2019generative,song2020score,yang2023diffusion, cao2024survey}. Let $x_0 \in \mathbb{R}^d$ denote a real data sample drawn from the data distribution $q\left(x_0\right)$, where $d$ is the dimensionality of the data.
The forward diffusion process constructs a Markov chain $\left\{x_t\right\}_{t=0}^T$ by gradually adding noise:

$$
x_0 \rightarrow x_1 \rightarrow \cdots \rightarrow x_T, \quad x_t \sim q\left(x_t \mid x_{t-1}\right), \quad t=1, \ldots, T,
$$ where each transition distribution $q\left(x_t \mid x_{t-1}\right)$ is a predefined forward-noising kernel.
Under the Markov assumption, the conditional density factorizes as $q\left(x_{1: T} \mid x_0\right)=\prod_{t=1}^T q\left(x_t \mid x_{t-1}\right).
$ The samples $x_0 \sim q\left(x_0\right)$ then be gradually corrupted by adding noise. A common choice of the noise is a Gaussian distribution: $q\left(x_t \mid x_{t-1}\right)=\mathcal{N}\left(x_t ; \sqrt{1-\beta_t} x_{t-1}, \beta_t \mathbf{I}\right)$ where the variance schedule $\beta_t \in(0,1)$ is fixed in advance, which in close form as: $x_t=\sqrt{1-\beta_t} x_{t-1}+\sqrt{\beta_t} \epsilon_{t-1}$ and $ \epsilon_{t-1} \sim \mathcal{N}(0, \mathbf{I})$. 

Then, to synthesize new data, the diffusion is reversed. Starting from noise corrupted data $x_T$ with prior $p\left(x_T\right)=\mathcal{N}\left(x_T ; 0, \mathbf{I}\right)$, the diffusion learns a new conditional distribution $p_\theta\left(x_{t-1} \mid x_t\right)$ with parameters $\theta$. Then,  the joint distribution becomes $p_\theta\left(x_{0: T}\right)=p\left(x_T\right) \prod_{t=1}^T p_\theta\left(x_{t-1} \mid x_t\right)$. The parameter $\theta$ is trained via a neural network loss that estimates the noise $\epsilon_\theta$ at each time step $t$: $L=\left\|\epsilon_t-\epsilon_\theta\left(x_t, t\right)\right\|^2$. \citep{song2020denoising} also provided an alternative sampling procedure while sharing the same training objective,  which constructs a deterministic (or optionally semi-stochastic) mapping that reuses the same noise-prediction network $\epsilon_\theta$ but allows much faster sampling, and the prediction of the sample is written as:

$$
\hat{x}_0\left(x_t\right)=\frac{x_t-\sqrt{1-\bar{\beta}_t} \epsilon_\theta\left(x_t, t\right)}{\sqrt{\bar{\beta}_t}}
$$ where $\bar{\beta}_t=\prod_{t=1}^T \beta_t$. Then the deterministic update is:

$$
x_{t-1}=\sqrt{\bar{\beta}_{t-1}} \hat{x}_0+\sqrt{1-\bar{\beta}_{t-1}} \frac{x_t-\sqrt{\bar{\beta}_t} \hat{x}_0}{\sqrt{1-\bar{\beta}_t}}
$$

The diffusion model objective admits an equivalent interpretation from a continuous-time perspective via score-based generative modeling \citep{song2021maximum,lim2025score}. 
The forward diffusion process can be described by the stochastic differential equation (SDE) $$dx_t=f(x_t,t)\,dt+g(t)\,dw_t$$ where $f(x_t,t)$ is the drift term, $g(t)$ is the diffusion coefficient, and $w_t$ denotes a standard Wiener process. This SDE induces a marginal distribution
$p_t(x_t)=\int q(x_t\mid x_0)q(x_0)\,dx_0$.

The corresponding reverse-time SDE is given by

$$dx_t=\big(f(x_t,t)-g(t)^2\,\mathbf{s}_t(x_t)\big)\,dt+g(t)\,dw_t$$
where $\mathbf{s}_t(x_t)=\nabla_{x_t}\log p_t(x_t)$ denotes the score function.
Since the marginal score is intractable, a neural network $\mathbf{s}_\theta(x_t,t)$ is trained to approximate the conditional score $\nabla_{x_t}\log q(x_t\mid x_0)$ using the denoising score matching objective:
\begin{equation}
\mathcal{L}(\theta)=
\mathbb{E}_{t,x_0,x_t\sim q(x_t\mid x_0)}
\Big[\lambda(t)\big\|\mathbf{s}_\theta(x_t,t)-\nabla_{x_t}\log q(x_t\mid x_0)\big\|^2\Big]
\label{eq:score_mat}
\end{equation}
where $\lambda(t)$ is a time-dependent weighting factor.

\subsection{Diffusion Policy with Reward Guidance}\label{sec:diffPol}

RL often involves high-dimensional and complex action and state spaces, which in turn induce complicated policies. This complexity makes it difficult for a simple sampling distribution to effectively capture or approximate the optimal policy. To address this challenge, diffusion models have been increasingly adopted to generate expressive policies for action sampling \citep{wang2022diffusion, zhang2023towards, kang2023efficient, ren2024diffusion}. In diffusion-policy notation, we use the superscript $k \in\{1, \ldots, K\}$ for diffusion steps and the subscript $t \in\{1, \ldots, T\}$ for policy time. Thus, the action at policy step $t$ and diffusion step $k$ is written as $a_t^k$. In this framework, for policy $\pi_\theta\left(\cdot \mid s_t\right)$ at time $t$ with parameter $\theta$ , trajectories of actions $a_t^{0: K}$ can be produced through the denoising process of the diffusion model as the descriptions in Sec~\ref{sec:diff}: $$
\pi_\theta(a_t^0 \mid s_t)=p_\theta\left(a_t^{0: K} \mid s_t\right)=\mathcal{N}\left(a_t^K ; 0, \mathbf{I}\right) \prod_{k=1}^K p_\theta\left(a_t^{k-1} \mid a_t^k, s_t\right)$$
Moreover, the original denoising process in diffusion models is inherently stochastic, and this randomness can slow down the training process in online RL by making it harder for the policy to converge to its optimal form. To mitigate this issue, many researchers have introduced guidance mechanisms that steer the diffusion model toward generating more optimal policies. For example, some scholars \citep{chen2022offline, kang2023efficient} use reward-guided diffusion, where the loss function is guided by $f\left(Q_\phi(s_t, a_t)\right)$ which is the  monotonically increasing function of the optimal reward $Q_\phi(s_t, a_t)$ with respect to a factor $\phi$ :

$$\mathcal{L}(\theta)=\mathbb{E}_{t, \mathbf{s}_t}\left[\beta_t \cdot f\left(Q_\phi(s_t, a_t)\right)\left\|\mathbf{s}_\theta\left(a_t^k, k\right)-\nabla_{a_t^k} \log q\left(a_t^k \mid a_t^0\right)\right\|^2\right]$$

The reward-guidance term $f\left(Q_\phi(s, a)\right)$ modifies the diffusion score so that the induced action-density is shifted toward high-value actions. Formally, the guided score becomes $\mathbf{s}_{\text {guided }}(a_t^0 \mid s_t)\coloneqq \mathbf{s}_\theta(a_t^0 \mid s_t)+\nabla_{a_t} f\left(Q_\phi(s_t, a_t)\right)$ which biases the denoising trajectory toward higher-reward directions. This reduces the stochastic drift of unguided diffusion and accelerates convergence toward better policies.

\subsection{Causal RL}\label{sec: cauRL}
Considering causality within the RL framework can strengthen the decision-making capability of the agent \citep{bareinboim2021introduction,seitzer2021causal, lu2022invariant, yu2023explainable, du2024situation,cao2025causal}. In a Markov decision process (MDP) $\mathcal{M}=\langle S, \mathcal{A}, P, R\rangle$, the interaction dynamics induce structured dependencies among the state space $S$, action space $\mathcal{A}$, transition probabilities $P$, and reward function $R$. Here, the transition probability captures the statistical associations—such as correlations or categorical dependencies—among states $S$, actions $\mathcal{A}$, and the reward function $R(s_t, a_t)$. However, an important open question is whether, after controlling for all confounding factors, altering a single coordinate of the action vector $a_t$—say $a \in a_t$ with $a_t \in \mathbb{R}^d$—actually causes subsequent states or rewards to change. Figure~\ref{fig:causal_asso} provides an illustration of a simple MDP, where the action, state, and reward are connected not only through statistical associations (blue edges) but also through causal relationships (red edges). 

Statistical association cannot separate the direct influence of an individual action element from effects created by other correlated components or hidden confounders. To uncover the true contribution of a particular $a$, we must \textbf{intervene} on the action variable within the MDP’s structural dynamics. This intervention allows us to construct a counterfactual trajectory, i.e., what the future states and rewards would have been if action $a$ had been implemented while keeping all other components of the system unchanged.

Compacting all variables at time $t$ into a latent set $\mathcal{D}_{t}=\left\{S_t, \mathcal{A}_t, r_t\right\}$, where $S_t \in \mathbb{R}^n$, $\mathcal{A}_t \in \mathbb{R}^d$ and $R \in \mathbb{R}$,  the causal dependencies among states, actions, next states, and rewards can be formalized using a Structural Causal Model (SCM). 
The SCM consists a directed acyclic graph (DAG) $\mathcal{G}=(\mathcal{D}_t, E)$ with edges $E$ encoding causal relations such as

$$
\left(s_t, a_t\right) \rightarrow s_{t+1}, \quad\left(a_t, s_{t+1}\right) \rightarrow r_t
$$ where $s_t \in S_t$, $a_t \in \mathcal{A}_t$. Under a set of structural conditionals $P\left(\mathcal{D}_t \mid \mathrm{PA}\left(\mathcal{D}_t\right)\right)$ defined for each node $\mathcal{D}_t \in \mathcal{D}$, where $\mathcal{D}= \left\{\mathcal{D}_t\right\}_{t=1}^T$ denotes the collection of variables across all time steps, the joint distribution over all causal variables factorizes as

$$
p\left(\mathcal{D}_1, \ldots, \mathcal{D}_T\right)=\prod_{t=1}^{T} p\left(\mathcal{D}_t \mid \operatorname{PA}\left(\mathcal{D}_t\right)\right)
$$ where $\mathrm{PA}\left(\mathcal{D}_t\right)$ denotes the set of parent variables in the associated DAG $\mathcal{G}$. The causal dependencies can be estimated as the factor over the distribution between parents and ancestors in the DAG by widely causal discovery methods such as PC (Peter–Clark) algorithm, NOTEARS, LiNGAM \citep{harris2013pc,zheng2018dags, shimizu2014lingam} 

By examining these underlying causal mechanisms, we can determine which components of the action vector meaningfully influence future states and high-reward behaviors before the action is actually executed. This provides a feasible and informative guidance signal for the diffusion policy.

\begin{figure}[H]
    \centering
    \includegraphics[width=0.5\linewidth]{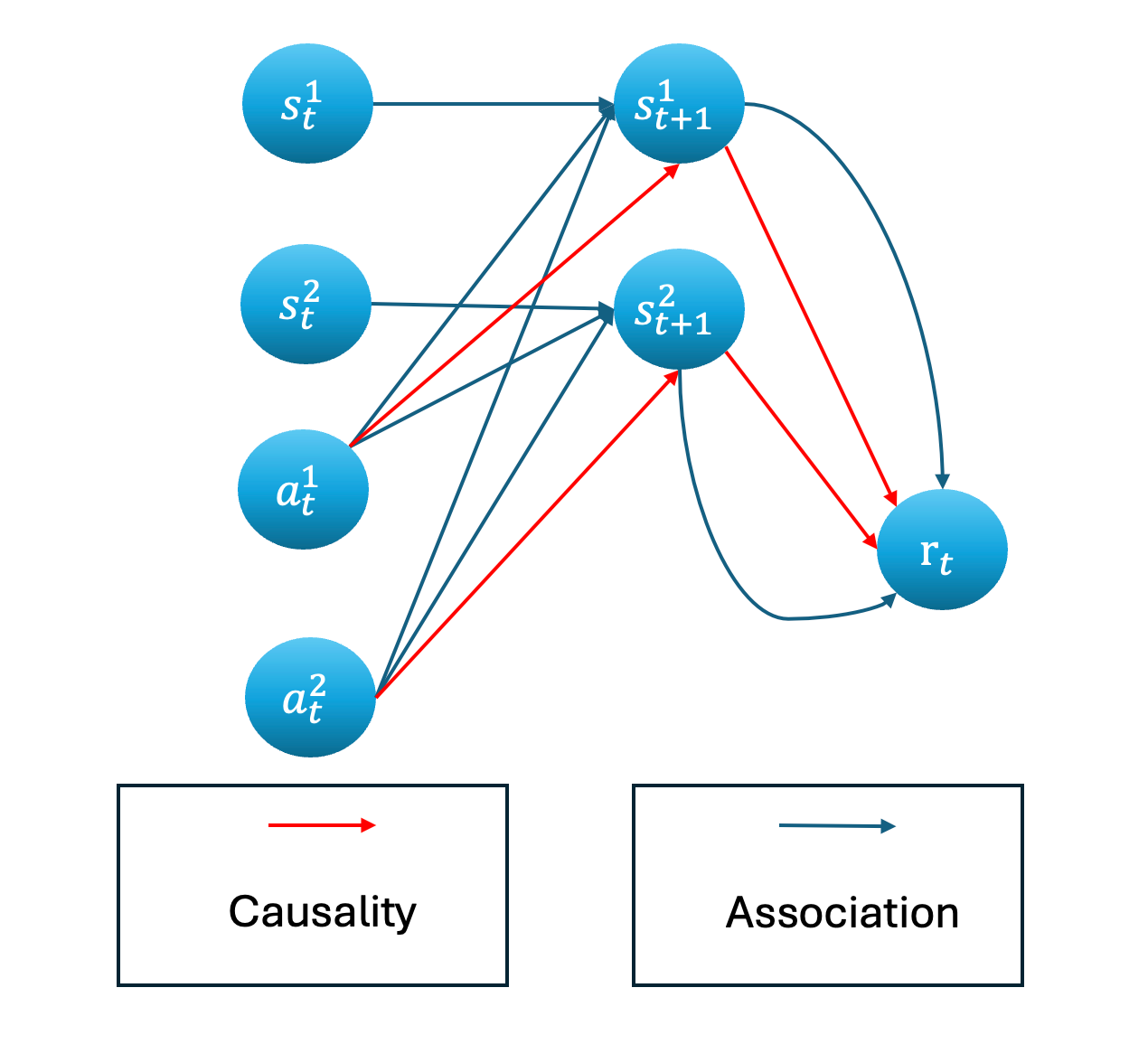}
    \caption{Causality and Association illustration}
    \label{fig:causal_asso}
\end{figure}

\section{Methodology}\label{sec:method}

The central contribution of our approach is the integration of causal guidance into the diffusion policy framework, enabling the diffusion model to generate policy that are aligned with the underlying causal dynamics of reinforcement learning. This causal guidance enables interventions on the actions then isolating the direct causal influence of the actions. Formally, applying an intervention replaces the generative mechanism of the action node is denoted as $\mathbf{d o}\left(a_t\right)$. Within diffusion policies, such intervention offers principled guidance to steer the denoising process toward actions that produce stronger causal impacts on future states and, ultimately, on the final rewards. 

Our method consists of two stages: \textbf{(1) Offline Stage.} We first construct initial causal dynamical models that encode the causal relationships among states, actions, and rewards using learned causal masks. These models specify parameterized structural equations for the state-transition and reward mechanisms. Because this serves as the offline initialization, the resulting causal masks and model parameters are stored and later used as priors for real-time causal adaptation.
\textbf{(2) Real-time Stage.} During interaction with the environment, new set of states, actions, and rewards are collected. Causal discovery methods are then applied to update the causal masks, enabling the causal dynamical models to adapt to the evolving environment.
These refined causal models serve as guidance distributions for the diffusion policy: they modulate the denoising process so that sampled actions are biased toward causally effective directions. As training progresses, both the diffusion policy and the causal masks are iteratively updated until the policy produces actions that exert the most beneficial causal influence toward achieving optimal rewards.

\subsection{Causal dynamical models }\label{sec:cauDM}

The offline dataset $\mathcal{D}$ consists of latent variables collected at each time step $t$, where each sample is $\mathcal{D}_{t}=\left\{S_t, \mathcal{A}_t, r_t\right\} \in \mathcal{D}$. Given this dataset, we estimate the underlying causal dependencies by learning a Directed Acyclic Graph (DAG) over the variables in $\mathcal{D}_t$ as  the joint distribution described in Sec~\ref{sec: cauRL}: $p\left(\mathcal{D}_1, \ldots, \mathcal{D}_T\right)=\prod_{t=1}^{T} p\left(\mathcal{D}_t \mid \operatorname{PA}\left(\mathcal{D}_t\right)\right)$. In this paper, we employ NOTEARS \citep{zheng2018dags} as our causal discovery method to estimate the causal structure for  straightforward implementation. It is worth emphasizing that our framework is not restricted to NOTEARS; it is fully compatible with any other reasonable and scalable causal discovery approaches. 

To encode the learned causal relations, we convert the discovered DAG structure into continuous causal masks, where each mask entry takes values in the interval $[0,1]$: a value close to $0$ indicates a weak or negligible causal influence, while a value close to $1$ indicates a strong causal dependency. Formally, let $$C_{s s}, C_{a s}:[0,1]^{n \times n},[0,1]^{n \times d}, \quad U_{s r}, U_{a r} \in[0,1]^n,[0,1]^d$$ denote the causal masks for the state-to-state and action-to-state transition mechanisms, and the state-to-reward and action-to-reward mechanisms, respectively. Equivalently, in terms of causal arrows:

$$
\left(s_t, a_t\right) \xrightarrow{C_{s s}, C_{a s}} s_{t+1}, \quad\left(a_t, s_{t+1}\right) \xrightarrow{U_{s r}, U_{a r}} r_t .
$$

The masks define which components of $s_t, a_t$ and $s_{t+1}$ are permitted to exert causal influence in the structural equations, and they are directly derived from the learned adjacency structure of the underlying DAG. Furthermore, these causal masks are applied as element-wise selectors on the corresponding states and actions, enabling the construction of causal dynamical models $f_{\varphi}(\cdot)$ and $g_{\omega}(\cdot)$ for their transitions:

$$
s_{t+1}=f_{\varphi}\left(C_{s s} \odot s_t, C_{s a} \odot a_t\right), \quad r_t=g_{\omega}\left(U_{s r} \odot s_{t+1}, U_{a r} \odot a_t\right)
$$ with parameters $\varphi$ and $\omega$. The causal dynamical model employs the Hadamard product $\odot$ to enforce a causal mask that is consistent with the actions and the state.

Importantly, the causal dynamical model can be instantiated as a probabilistic transition density. For scalability, we adopt a Gaussian parameterization following prior work
\citep{wang2022causaldynamicslearningtaskindependent, feng2023learningdynamicattributefactoredworld}.
Specifically, the state transition and reward models factorize as: 
\begin{equation}
\begin{aligned}
p\left(s_{t+1} \mid s_t, a_t\right) & =\mathcal{N}\left(f_{\varphi}\left(C_{s s} \odot s_t, C_{s a} \odot a_t\right), \Sigma_{\varphi}\right), \\
p\left(r_t \mid s_{t+1}, a_t\right) & =\mathcal{N}\left(g_\omega\left(U_{s r} \odot s_{t+1}, U_{a r} \odot a_t\right), \Sigma_\omega\right)
\end{aligned}\label{eq:linearmap}
\end{equation} with the associated covariance matrix $\Sigma_{\varphi} $ and $\Sigma_\omega$. 
The Gaussian probabilistic formulation makes the causal model naturally compatible with diffusion policies, whose action generation is also based on Gaussian sampling. As a result, the causal structure can be incorporated as a guidance signal during the diffusion process without introducing any mismatch in the underlying sampling dynamics.

Meanwhile, the diffusion policy is parameterized by an initial noise prediction network $\epsilon_\theta\left(a_t, t\right)$ which models the denoising transitions at time $t$ as
\begin{equation}\label{eq:diffuPol}
\pi_\theta\left(a_t^0 \mid s_t\right)=p_\theta\left(a_t^{0: K} \mid s_t\right)=\mathcal{N}\left(a_t^K ; 0, \mathbf{I}\right) \prod_{i=1}^K p_\theta\left(a_t^{i-1} \mid a_t^i, s_t\right) 
\end{equation}

This network $\epsilon_\theta\left(a_t, t\right)$ is first trained offline using the dataset $\mathcal{D}$ by standard loss that estimates the noise $\epsilon_\theta$ : $L=\left\|\epsilon_t-\epsilon_\theta\left(a_t, t\right)\right\|^2$ to learn an accurate score (noise estimation) model for action sampling. During online interaction, $\epsilon_\theta$ will be further updated based on environmental feedback and the causal guidance signal for continual policy improvement.

\subsection{Real-time Causality-Guided Diffusion Policy}\label{sec:CGdiffup}

Causal relationships are not static among the components in RL. At every time step, new actions may causally influence the updated states, which in turn produce new rewards. Thus, each new action can be regarded as an \textbf{intervention} on the evolving causal relationships, denoted as $\operatorname{\textbf{do}}(a_t)$ at each time step.

By the Global Markov Condition\citep{pearl2009causality, janzing2010causal}, $\operatorname{\textbf{do}}(a_t)$ has a causal effect on the state and reward  and the causal effect on state and reward is independently conditional on the action $\mathcal{A}$. Therefore, the causal dynamical model can be integrated as:

\begin{equation}
p\left(s_{t+1}, r_t \mid s_t, \operatorname{\textbf{do}}\left(a_t\right)\right)=p\left(s_{t+1} \mid s_t, \operatorname{\textbf{do}}\left(a_t\right)\right) p\left(r_t \mid s_{t+1}, \operatorname{\textbf{do}}\left(a_t\right)\right)
\label{eq:jointPro}
\end{equation}
 which means that a different action $a_t$ (viewed as the intervention $\operatorname{\textbf{do}}\left(a_t\right)$ ) can causally influence the next state $s_{t+1}$ and the corresponding reward $r_t$. This causal property enables the causal probabilistic model to dynamically guide the diffusion policy to sample actions that achieve optimal rewards. Meanwhile, the intervention $\operatorname{\textbf{do}}\left(a_t\right)$ is used to refine the causal masks, which in turn enables the guidance to be dynamically updated during the diffusion policy generation.

As we mentioned in Sec~\ref{sec:diff} and Sec~\ref{sec:diffPol}, in diffusion-based policies, action generation is modeled as a learned reverse denoising process conditioned on the state $s$. The reverse process with a parameterized $\theta$ can be represented as:

$$
p_\theta\left(a^{k-1} \mid a^k, s\right)=\mathcal{N}\left(a^{k-1} ; \mu_\theta\left(a^k, s\right), \Sigma_k\right)
$$ which defines the joint density $p_\theta\left(a^{0: K} \mid s\right)=p\left(a^{K}\right) \prod_{k=1}^K p_\theta\left(a^{k-1} \mid a^{k}, s\right)$.

The policy corresponds to the marginal at the final denoising step:

$$
\pi_\theta\left(a^{0} \mid s\right)=p_\theta\left(a^{0} \mid s\right)=\int p_\theta\left(a^{0: K} \mid s\right) d a^{1: K}
$$ Given network predicts the injected noise $\epsilon_\theta\left(a^k, k\right)$, the mean of the distribution in reverse process can be rewritten as:

$$
\mu_\theta\left(a^{k-1}, s, k\right)=\frac{1}{\sqrt{\beta_k}}\left(a^k-\frac{1-\beta_k}{\sqrt{1-\bar{\beta}_k}} \epsilon_\theta\left(a^k, k\right)\right)
$$ where $
\bar{\beta}^k=\prod_{k=1}^K \beta_k
$.

Importantly, the diffusion policy cannot rely solely on the denoising process, because an RL agent must sample actions that causally influence future states and rewards. Therefore, from the integrated causal dynamical model, \textbf{causal guidance} is introduced by differentiating the joint transition-reward density in Eq.~\ref{eq:jointPro}, yielding:

\begin{equation}
\begin{aligned}
\nabla_{a^k_t} \log p\left(s_{t+1}, r_t \mid s_t, \operatorname{\textbf{do}}\left(a^k_t\right)\right)= & \nabla_{a^k_t} \log p\left(s_{t+1} \mid s_t, \operatorname{\textbf{do}}\left(a^k_t\right)\right) \\
& +\nabla_{a^k_t} \log p\left(r_t \mid s_{t+1}, \operatorname{\textbf{do}}\left(a^k_t\right)\right)
\end{aligned}
\label{eq:mean_cg}
\end{equation}

This gradient of the causal density describes the direction in which the state and reward distributions shift when the action is intervened via $\operatorname{\textbf{do}}\left(a_t\right)$. Firstly, the causal masks factorize the causal and non-causal dependencies, strengthening the action dimensions that directly influence the transition and reward mechanisms. Consequently, the derived causal guidance informs the diffusion policy of the reward-improving direction and adjusts its denoising trajectory toward causally optimal reward from actions. Since the optimal reward $r^*$ attainable from each offline environment interaction is known in advance, we can use $r^*$ to condition the causal dynamical model as: 
$$\begin{aligned}
\nabla_{a^k_t} \log p\left(s_{t+1}, r^* \mid \cdot, \operatorname{\textbf{do}}\left(a^k_t\right)\right)= & \gamma_t\nabla_{a^k_t} \log p\left(s_{t+1} \mid \cdot, \operatorname{\textbf{do}}\left(a^k_t\right)\right) \\
& +\beta_t\nabla_{a^k_t} \log p\left(r^* \mid \cdot, \operatorname{\textbf{do}}\left(a^k_t\right)\right)
\end{aligned}$$
 where $\gamma_t$ and $\beta_t$ are guidance coefficients balancing state and reward causality. In this way, the diffusion policy is guided by a causally informed mean 
$\mu_\theta^{\mathrm{cg}}$, defined as

$$\mu_\theta^{\mathrm{cg}}(a_t^k, s_t, t) 
= \mu_\theta(\operatorname{\textbf{do}}(a_t^k), s_t, t)$$

Accordingly, the original diffusion policy in Eq.~\ref{eq:diffuPol} is transformed
into a \textbf{causality-guided diffusion policy}, denoted by $\pi_\theta^{\mathrm{cg}}$. The score function with noise predictor  $\epsilon_\theta(a_t^k, k)$ of the noisy action distribution admits the following approximation 
$$\nabla_{a_t^k} \log p_k(a_t^k \mid s_t)
\approx
-\frac{1}{\sqrt{1-\bar{\alpha}_k}}\,
\epsilon_\theta(a_t^k, k)$$ where $\bar{\alpha}_k$ is cumulative product of the diffusion noise parameter. 

Causal guidance is then incorporated in the noise prediction of the diffusion policy. Specifically, the
causality-guided score is defined as

$$\nabla_{a_t^k} \log p_k^{\mathrm{cg}}(a_t^k \mid s_t)
=
\nabla_{a_t^k} \log p_k(a_t^k \mid s_t)
+
\lambda_k
\nabla_{a_t^k}
\log p\left(
s_{t+1}, r_t \mid s_t, \mathbf{do}(a_t^k)
\right)$$ where $\lambda_k$ is the coefficient. Mapping the guided score back to the noise parameterization yields the corresponding
causality-guided noise predictor:

$$\epsilon_\theta^{\mathrm{cg}}(a_t^k, k)
=
\epsilon_\theta(a_t^k, k)
-
\lambda_k \sqrt{1-\bar{\alpha}_k}\,
\nabla_{a_t^k}
\log p\left(
s_{t+1}, r_t \mid s_t, \mathbf{do}(a_t^k)
\right)$$



Given $\epsilon_\theta^{\mathrm{cg}}$, the estimate of the clean action is: $$
\hat{a}_0^{\mathrm{cg}}=\frac{a^k-\sqrt{1-\bar{\beta}_k} \epsilon_\theta^{\mathrm{cg}}}{\sqrt{\bar{\beta}_k}}
$$
Finally, the guided  update for the previous action sample is: $a^{k-1}=\sqrt{\bar{\beta}_{k-1}} \hat{a}_0^{\mathrm{cg}}+\sqrt{1-\bar{\beta}_{k-1}} \epsilon_\theta^{\mathrm{cg}}$.



After sampling actions from the causal-guided diffusion policy, we update the policy parameters by optimizing a loss that combines the original diffusion denoising objective with an actor-style objective defined by the Q-network $Q_\phi$.
Specifically, the causal guidance-augmented action generator is denoted as $G_\theta(s ; z)$, where $z \sim \mathcal{N}(0, \mathbf{I})$ is the initial sampling noise. The policy gradient is therefore given by:
\begin{equation}
\nabla_\theta \mathcal{L}_{\text{policy}}(\theta)
=
\nabla_\theta \mathbb{E}\!\left[\left\|\hat{\epsilon}_\theta\left(s, a^0_t, t\right)-\boldsymbol{\epsilon}_{\theta}^{\mathrm{cg}}\right\|^2\right]
-\lambda \nabla_\theta \mathbb{E}\!\left[Q_\phi\left(s, G_\theta(s ; z)\right)\right]
\label{eq:policyUpd}
\end{equation}

We adopt the double $Q$-network \citep{hasselt2010double, van2016deep, chen2021randomized} to learn the $Q_\phi$. For the two Q-networks $Q_{\phi_1}, Q_{\phi_2}$ and their corresponding target networks $Q_{\phi_1^{\prime}}, Q_{\phi_2^{\prime}}$, we define the temporal difference (TD) target as the minimum of the two target Q-values:

\begin{equation}
y=r+\gamma \min \left(Q_{\phi_1^{\prime}}\left(s_{t+1}, a^0_t\right), Q_{\phi_2^{\prime}}\left(s_{t+1}, a^0_t\right)\right)
\label{eq:TDerr}
\end{equation} Then the combined double Q-network loss can be expressed as: 

\begin{equation}
    \mathcal{L}\left(\phi_1, \phi_2\right)=\mathbb{E}_{\left(a^0_t, s_{t+1}\right)}\left[\left(Q_{\phi_1}(s_{t+1}, a^0_t)-y\right)^2+\left(Q_{\phi_2}(s_{t+1}, a^0_t)-y\right)^2\right]
\label{eq:doubleQ}
\end{equation}

Overall, the proposed causal guidance is architecture-agnostic: it integrates seamlessly with various diffusion policies and scales to tasks of different complexity.

We summarize the proposed causality-guided diffusion policy in Algorithm \ref{alg:1}.

\begin{algorithm}[H]
\caption{Causality-Guided Diffusion Policy }
\begin{algorithmic}[1]

\STATE \textbf{Part 1: Offline Causal Models}
\STATE \textbf{Inputs:} dataset $\mathcal{D}$

\STATE Initialize model parameters $f_\varphi, g_\omega$
\FOR{$t = 1$ to $N$}
    \STATE Sample $\mathcal{D}_t=\left\{S_t, \mathcal{A}_t, r_t\right\}$ from dataset $\mathcal{D}$

    \STATE Compute causal masks $C_{s s}, C_{a s}, U_{s r}, U_{a r} $ by the DAG of $\mathcal{D}$ as Sec~\ref{sec:cauDM}
    \STATE Train initial causal dynamic model $f_{\varphi}(\cdot)$ and $g_{\omega}(\cdot)$ following Sec~\ref{sec:cauDM}

\ENDFOR

\vspace{15pt}

\STATE \textbf{Part 2: Real-time Causal-Guided Diffusion}
\STATE \textbf{Inputs:} Trained models:  $f_\varphi, g_\omega$, environment $\mathcal{E}$, and initial Causal-guided policy $\pi_\theta^{\mathrm{cg}}$

\FOR{$t = 1$ to $T$}
    \STATE Observe  $\{a_t,r_t,s_t,s_{t+1}\}$
    \STATE Sampling initial $a^0_{t} \sim \pi_\theta^{\mathrm{cg}}$ and learning real-time causal dynamical model (Eq~\ref{eq:mean_cg}) \textbf{for} $k= 1$ to $K$: \[
            \gamma_t \nabla_{a_t^k}\log p\big(s_{t+1}\mid s_t,\operatorname{\textbf{do}}(a_t^k)\big)
            + \beta_t \nabla_{a_t^k}\log p\big(r_t\mid s_{t+1},\operatorname{\textbf{do}}(a_t^k)\big)
            \]

         
        \STATE Updating the causal-guided diffusion policy as Eq.~\ref{eq:policyUpd}
        
        \STATE Updating the $Q$-value by interacting with $\mathcal{E}$ following Eqs.~\ref{eq:TDerr}, and \ref{eq:doubleQ}

    \ENDFOR

\end{algorithmic}
\label{alg:1}
\end{algorithm}

\section{Theoretical analysis} \label{sec:theory}

\subsection{Step Size Stability}

To ensure that the proposed causality-guided diffusion policy is numerically and theoretically feasible, we analyze its dynamics under the stochastic differential equation (SDE) formulation. Specifically, we will show that the causality-guided diffusion model does not exhibit unbounded step sizes. 

A common diffusion model can be represented as a SDE \citep{song2020score}. The forward term for the action $a_t$ can be written as:

$$d a_t=f\left(a_t, t\right) d t+g(t) d \bar{w}_t$$ 
where $f\left(a_t, t\right)=-\frac{1}{2} \beta(t)\left(a_t-\mu_\theta\left(a_t, t\right)\right)$, with  $\mu_\theta(a_t, t)$ denoting a parameterized mean with parameters $\theta$, $d \bar{w}_t$ is a Wiener process (Brownian motion), and $g(t)=\sqrt{\beta(t)}$. The reverse process corresponds to denoising can be written as:

$$d a_t=\left[f\left(a_t, t\right)-g(t)^2 \nabla_{a_t} \log p_t\left(a_t\right)\right] d t+g(t) d \bar{w},$$ 
where $p_t\left(a_t\right)$ is the sampling distribution. 

Under our causality-guided diffusion, the SDE of the diffusion process is:

\begin{equation}
\begin{aligned}
d a_t=[f & \left(a_t, t\right)-g(t)^2 \nabla_{a_t} \log p_t\left(a_t\right) \\
& -\gamma_t \nabla_{a_t} \log p_{\varphi}\left(s_{t+1} \mid s_{t}, \operatorname{\textbf{do}}\left(a_t\right)\right) \\
& \left.-\beta_t \nabla_{a_t} \log p_{\omega}\left(r^* \mid s_{t+1}, \operatorname{\textbf{do}}\left(a_t\right)\right)\right] d t+g(t) d \bar{w}_t.
\end{aligned}
\label{eq:diffSDE}
\end{equation}

Under the SDE formulation of the causality-guided diffusion policy, we introduce a set of basic and mild assumptions in Assumption~\ref{assump:lipschitz} to ensure that the causality-guided diffusion policy is well-defined and valid.

\newtheorem{assumption}{Assumption}\label{ass:assump}

\begin{assumption}[Lipschitz Regularity of the causality guided Drift]\label{assump:lipschitz}
There exist constants $L_f, L_s, L_{\varphi},L_\omega \ge 0$ such that for any $a, a' \in \mathbb{R}^d$ and any time $t \in [0, T]$, the following Lipschitz conditions hold:
\begin{align}
&\| f(a, t) - f(a', t) \| \le L_f \| a - a' \|, \tag{A1}\label{ass:A1} \\
&\| \nabla_a \log p_t(a) - \nabla_{a'} \log p_t(a') \| \le L_s \| a - a' \|, \tag{A2}\label{ass:A2} \\
&\| \nabla_a \log p_{\varphi}(\cdot \mid \operatorname{\textbf{do}}(a)) - \nabla_{a'} \log p_{\varphi}(\cdot \mid \operatorname{\textbf{do}}(a')) \| \le L_{\varphi} \| a - a' \|, \tag{A3}\label{ass:A3}\\
&\| \nabla_a \log p_{\omega}(\cdot \mid \operatorname{\textbf{do}}(a)) - \nabla_{a'} \log p_{\omega}(\cdot \mid \operatorname{\textbf{do}}(a')) \| \le L_{\omega} \| a - a' \|
\tag{A4}\label{ass:A4}
\end{align}
In other words, the base drift term $f(a,t)$, the diffusion score $\nabla_a \log p_t(a)$, and the causal guidance terms $\nabla_a \log p_{\varphi}$ and $\nabla_a \log p_\omega$ are all globally Lipschitz continuous in $a$ with respective Lipschitz constants $L_f$, $L_s$, and $L_{\varphi}$.
\end{assumption}

Assumptions \ref{ass:A1} and \ref{ass:A2} are the standard Lipschitz continuity and linear growth conditions for the drift term of an SDE \citep{oksendal2013stochastic,karatzas2014brownian}. These two assumptions guarantee the \textit{\textbf{existence and uniqueness}} of the SDE solution. In our method, this implies that the action trajectory generated by the diffusion policy is also well-defined and unique. This matches the RL setting, where the agent cannot execute two different actions at the same time at a given state. Similarly, in Assumptions \ref{ass:A3} and \ref{ass:A4}, the two guidance terms are introduced into the diffusion model in a way that ensures the resulting SDE remains mathematically well-posed and numerically stable \citep{song2019generative, oh2024stable}. These two mild assumptions ensure that the RL agent does not generate multiple inconsistent actions for the same state; otherwise, the policy may become undefined.

Based on these assumptions, we define a sum Lipschitz factor as $L_{\text {sum }}(t)=L_f+g(t)^2 L_s+\left|\gamma_t\right| L_{\varphi}+\left|\beta_t\right| L_{\omega}$. Then, we can have the following proposition: 
\newtheorem{prop}{Proposition} 
\begin{prop}\label{pro:pro1}
Consider the causal-guided reverse diffusion SDE~\eqref{eq:diffSDE} under Assumption~\ref{assump:lipschitz}, the explicit Euler discretizations $a_{n+1}=a_n+b_{\text {guided }}\left(a_n, t_n\right) \Delta t+g\left(t_n\right) \Delta \bar{w}_n
$ is stable as the time step $\Delta t$ satisfies $\Delta t \leq \frac{\delta}{L_{\text {sum }}\left(t_n\right)}$ for some $0<\delta<1$.
Equivalently, a sufficient stability constraint is

$$
\boxed{
\Delta t \leq \frac{\delta}{L_f+g(t)^2 L_s+\left|\gamma_t\right| L_{\varphi}+\left|\beta_t\right| L_{\varphi}}}.
$$

Under this step-size condition, the Euler solution remains stable and the discretization error stays bounded.
\end{prop}

Proposition~\ref{pro:pro1} guarantees that introducing causality-guidance into diffusion policies does not compromise the underlying diffusion model, while ensuring that the resulting SDE remains stable under reasonable step sizes.

\begin{proof}[\textbf{Proof of Proposition 1}]
We now analyze stability under the (explicit) Euler method. The drifted causility-guided term is dnoted as :$$
\begin{aligned}
b_{\text{cg}}(a_t, t)
= {} & f(a_t, t)
- g(t)^2 \, \nabla_{a_t} \log p_t(a_t) \\
& - \gamma_t \, \nabla_{a_t} 
    \log p_{\varphi}\!\left(
        s_{t+1} \mid s_t, \operatorname{\mathbf{do}}(a_t)
    \right) \\
& - \beta_t \, \nabla_{a_t}
    \log p_{\omega}\!\left(
        r^* \mid s_{t+1}, \operatorname{\mathbf{do}}(a_t)
    \right)
\end{aligned}
$$

In the Euler method, the numerical approximation at each step is discretized as $t \to t_n$ as $t_n=n \Delta t, \quad n=0,1, \ldots, N$ by $\Delta t$ gap for each discretized gap. Then, the SDE is discretized as: 

$$a_{n+1}=a_n+b_{\text{cg}}\left(a_n, t_n\right) \Delta t+g\left(t_n\right) \Delta \bar{w}_n$$ where $\Delta \bar{w}_n \sim \mathcal{N}(0, \Delta t \mathbf{I})$ and $t_{n+1}=t_n+\Delta t$. According to the exact solution expansion:
$$a\left(t_{n+1}\right)=a\left(t_n\right)+b_{\text{cg}}\left(a\left(t_n\right), t_n\right) \Delta t+R_n+g\left(t_n\right) \Delta \bar{w}_n$$ where $R_n$ is the local truncation (remainder) error, the error between the true and numerical solutions can then be written as $e_n=a\left(t_n\right)-a_n$and it satisfies the following recursion:

$$
e_{n+1}=e_n+\left(b_{\mathrm{cg}}\left(a\left(t_n\right), t_n\right)-b_{\mathrm{cg}}\left(a_n, t_n\right)\right) \Delta t+R_n
$$ 

By Assumption \ref{ass:assump}, 
, when $n=0$, we can have $\left\|b_{\mathrm{cg}}\left(a\left(t_0\right), t_0\right)-b_{\mathrm{cg}}\left(a_0, t_0\right)\right\| \leq L_{\text {sum }}\left(t_0\right)\left\|a\left(t_0\right)-a_0\right\|$, so $\left\|e_1\right\| \leq\left(1+\Delta t L_{\mathrm{sum}}(t_0)\right)\left\|e_0\right\|+\left\|R_0\right\|$. This can be deduced for $m$ when $1<m<N$ as: $$\left\|e_m\right\| \leq\left(\prod_{k=0}^{m-1}\left(1+\Delta t L_{\mathrm{sum}}(t_k)\right)\right)\left\|e_0\right\|+\sum_{j=0}^{m-1}\left(\prod_{k=j+1}^{m-1}\left(1+\Delta t L_{\mathrm{sum}}(t_k)\right)\right)\left\|R_j\right\|$$ which gives us the recursion:
$$\begin{aligned}\left\|e_{m+1}\right\| & \leq\left(1+\Delta t L_m\right)\left[\left(\prod_{k=0}^{m-1}\left(1+\Delta t L_{\mathrm{sum}}(t_k)\right)\right)\left\|e_0\right\|+\sum_{j=0}^{m-1}\left(\prod_{k=j+1}^{m-1}\left(1+\Delta t L_{\mathrm{sum}}(t_k)\right)\right)\left\|R_j\right\|\right]+\left\|R_m\right\| \\ & =\left(\prod_{k=0}^m\left(1+\Delta t L_{\mathrm{sum}}(t_k)\right)\right)\left\|e_0\right\|+\sum_{j=0}^{m-1}\left(\prod_{k=j+1}^m\left(1+\Delta t L_{\mathrm{sum}}(t_k)\right)\right)\left\|R_j\right\|+\left\|R_m\right\|\end{aligned}$$ Finally, the error at step $N$ can be bounded as:

$$\left\|e_N\right\| \leq\left(\prod_{k=0}^{N-1}\left(1+\Delta t L_{\text {sun }}\left(t_k\right)\right)\right)\left\|e_0\right\|+\sum_{j=0}^{N-1}\left(\prod_{k=j+1}^{N-1}\left(1+\Delta t L_{\text {sum }}\left(t_k\right)\right)\right)\left\|R_j\right\|$$
This recurrence relation forms the basis for analyzing the stability of the Euler discretization in subsequent steps, which implies the Lipschitz bound:

$$\begin{aligned}
\|e_{n+1}\|
&\le \|e_n\|
   + \left\| b_{\mathrm{cg}}(a(t_n), t_n)
     - b_{\mathrm{cg}}(a_n, t_n) \right\| \Delta t
   + \|R_n\| \\[6pt]
&\le \|e_n\|
   + L_{\mathrm{sum}}(t_n)\, \| a(t_n) - a_n \| \, \Delta t
   + \|R_n\| \\[6pt]
&= \|e_n\| + \Delta t\, L_{\mathrm{sum}}(t_n)\, \|e_n\| + \|R_n\| \\[6pt]
&= \left(1 + \Delta t L_{\mathrm{sum}}(t_n)\right)\|e_n\| + \|R_n\|
\end{aligned}
$$
Moreover, to control the error in the recursion, we require the factor $1+\Delta t L_n$ to be bounded by a constant $1+\delta$ with $\delta>0$. This implies that $1+\Delta t L_{\text {sum }}\left(t_N\right) \leq 1+\delta
$ for some sufficiently small $\delta>0$.

Therefore, the time step for ensuring the stability of the causality-guided diffuion model is 

$$\Delta t \leq \frac{\delta}{L_{\text {sum}}\left(t_N\right)}$$ which can be written explicitly as

$$\Delta t \leq \frac{\delta}{L_f+g(t)^2 L_s+\left|\gamma_t\right| L_{\varphi}+\left|\beta_t\right| L_{\varphi}}$$

This timestep condition is sufficient to ensure the stability of  the causality-guided diffusion policy with respect to the action
 $a_t$. 
\end{proof}

\subsection{Performance Difference}

Moreover, we aim to quantify how the introduction of causal guidance affects the behavior and performance of the diffusion policy. To this end, we establish a theoretical guarantee characterizing the performance difference between the causality-guided policy and the base diffusion policy using the Performance Difference Lemma (PDL). The PDL expresses the difference in performance between two policies in terms of the expectation of the advantage function \citep{kakade2002approximately}. Specifically, for any initial state $s_0$, the performance gap between two policies $\pi$ and $\pi^{\prime}$ is given by

$$
V^\pi\left(s_0\right)-V^{\pi^{\prime}}\left(s_0\right)=\frac{1}{1-\gamma} \mathbb{E}_{s \sim d_{s_0}^\pi}\left[\mathbb{E}_{a \sim \pi(\cdot \mid s)} A^{\pi^{\prime}}(s, a)\right],
$$ where $V^\pi\left(s_0\right)=\mathbb{E}\left[\sum_{i=0}^{\infty} \gamma^i r\left(s_i, a_i\right) \mid s_0, \pi\right]$ is the discounted return (value function) of policy $\pi$, and $A^{\pi^{\prime}}(s, a)=Q^{\pi^{\prime}}(s, a)-V^{\pi^{\prime}}(s)$ is the advantage function under policy $\pi^{\prime}$ with transition probability $d = p\left(s_{t+1} \mid s_t, a_t\right)$. Thus, we construct our Theorem~\ref{theor: PDLcau} below.

\newtheorem{theorem}{Theorem}

\begin{theorem}\label{theor: PDLcau}
Let $\pi$ be the base (unguided) diffusion policy and $\pi_{cg}$ be the causality-guided diffusion policy obtained by adding drift terms with interventions $$\gamma_t \,\nabla_{a_t}\log p_{\varphi}(\cdot\mid\operatorname{\mathbf{do}}(a_t))
+
\beta_t \,\nabla_{a_t}\log p_{\omega}(\cdot\mid\operatorname{\mathbf{do}}(a_t))$$ Assume the advantage function $A^\pi(s,a)$ is finite, then the performance difference satisfies the bound

$$\begin{aligned}
&\big|J(\pi_{cg})-J(\pi)\big|\\
&\le \frac{1}{1-\gamma}\,
\sqrt{\mathbb{E}_{s\sim d^{\pi_{cg}}}\!\big[\sup_a |A^\pi(s,a)|^2\big]}
\cdot
\sqrt{\tfrac{1}{2}\,\mathbb{E}_{s\sim d^{\pi_{cg}}}\!\big[
\mathrm{KL}\!\big(\pi_{cg}(\cdot\mid s)\,\|\,\pi(\cdot\mid s)\big)\big]}\\
&=
\frac{1}{1-\gamma}\,
\sqrt{\mathbb{E}_{s\sim d^{\pi_{cg}}}\!\big[\sup_a |A^\pi(s,a)|^2\big]}
\cdot
\sqrt{\tfrac{1}{2}\,\mathbb{E}_{\pi_{cg}}\!\left[\int_0^T
\Big\|\frac{\gamma_t \nabla_{a_t}\log p_{\varphi}(\cdot\mid\operatorname{\mathbf{do}}(a_t))
+\beta_t \nabla_{a_t}\log p_{\omega}(\cdot\mid\operatorname{\mathbf{do}}(a_t))}{g(t)}\Big\|^2 dt\right]}
\end{aligned}$$
 which implies that $$\left|J\left(\pi_{c g}\right)-J(\pi)\right| \leq \mathcal{O}\left(\sqrt{\mathrm{KL}\left(\pi_{c g} \| \pi\right)}\right)$$
\end{theorem}

\begin{proof}[\textbf{Proof of Theorem 1}]

By the Performance Difference Lemma (PDL), we can first have $$J(\pi_{cg})-J(\pi)
= \frac{1}{1-\gamma}\,\mathbb{E}_{s\sim d^{\pi_{cg}}}\Big[\,\mathbb{E}_{a\sim\pi_{cg}(\cdot\mid s)}\big[A^\pi(s,a)\big]\,\Big]$$ where \(d^{\pi_{cg}}\) denotes the (discounted) state-visitation distribution under \(\pi_{cg}\), which implies that

$$\mathbb{E}_{a\sim\pi_{cg}}[A^\pi(s,a)]
= \mathbb{E}_{a\sim\pi_{cg}}[A^\pi(s,a)] - \mathbb{E}_{a\sim\pi}[A^\pi(s,a)]
= \mathbb{E}_{a}\!\big[A^\pi(s,a)\,(\pi_{cg}(a\mid s)-\pi(a\mid s))\big]$$ as \(A^\pi(s,a)=Q^\pi(s,a)-V^\pi(s)\) and \(V^\pi(s)=\mathbb{E}_{a\sim\pi}Q^\pi(s,a)\).
Hence, this PDL can be represented by the factorized expectation: 
$$J(\pi_{cg})-J(\pi)
= \frac{1}{1-\gamma}\,\mathbb{E}_{s\sim d^{\pi_{cg}}}\!\Big[
\mathbb{E}_{a}\!\big[A^\pi(s,a)\,(\pi_{cg}-\pi)(a\mid s)\big]\Big]$$

For each state \(s\), we can derive the upper bound of this expectation: 

$$\big|\mathbb{E}_{a}\!\big[A^\pi(s,a)\,(\pi_{cg}-\pi)(a\mid s)\big]\big|
\le \sup_{a}|A^\pi(s,a)| \cdot \int |\,\pi_{cg}(a\mid s)-\pi(a\mid s)\,|\,da$$

The integral on the right is twice the total variation (TV) distance:

$$\mathrm{TV}\!\big(\pi_{cg}(\cdot\mid s),\pi(\cdot\mid s)\big)
=\tfrac12\int |\pi_{cg}-\pi|(a\mid s)\,da$$

Hence, we can have $$\big|\mathbb{E}_{a}\!\big[A^\pi(s,a)\,(\pi_{cg}-\pi)(a\mid s)\big]\big|
\le 2\,\sup_{a}|A^\pi(s,a)| \cdot \mathrm{TV}\!\big(\pi_{cg}(\cdot\mid s),\pi(\cdot\mid s)\big)$$

By  Pinsker's inequality, for two probability measures $p$ and $q$, the total variation (TV) distance is bounded by the square root of half Kullback–Leibler (KL) divergence as: 
$$\mathrm{TV}(p,q) \le \sqrt{\tfrac{1}{2}\,\mathrm{KL}(p\|q)}$$

Applying this to the conditional policies at each $s$, we can drive that: 

$$\big|\mathbb{E}_{a}\!\big[A^\pi(s,a)\,(\pi_{cg}-\pi)(a\mid s)\big]\big|
\le 2\,\sup_{a}|A^\pi(s,a)| \cdot
\sqrt{\tfrac{1}{2}\,\mathrm{KL}\!\big(\pi_{cg}(\cdot\mid s)\,\|\,\pi(\cdot\mid s)\big)}$$ which yields
$$\begin{aligned}
\big|J(\pi_{cg})-J(\pi)\big|
&\le \frac{1}{1-\gamma}\,
\mathbb{E}_{s\sim d^{\pi_{cg}}}\!\Big[\,2\sup_{a}|A^\pi(s,a)|\,
\sqrt{\tfrac{1}{2}\,\mathrm{KL}\!\big(\pi_{cg}(\cdot\mid s)\,\|\,\pi(\cdot\mid s)\big)}\Big]\\
&= \frac{2}{1-\gamma}\,
\mathbb{E}_{s\sim d^{\pi_{cg}}}\!\Big[\sup_{a}|A^\pi(s,a)|\,
\sqrt{\tfrac{1}{2}\,\mathrm{KL}\!\big(\pi_{cg}(\cdot\mid s)\,\|\,\pi(\cdot\mid s)\big)}\Big]
\end{aligned}$$

According to the Cauchy--Schwarz of the two nonnegative factors $u(s),v(s)$ under the state expectation:

$$\mathbb{E}_s\big[u(s)\,v(s)\big] \le \sqrt{\mathbb{E}_s[u(s)^2]}\;\sqrt{\mathbb{E}_s[v(s)^2]}$$

Set $u(s)=\sup_a |A^\pi(s,a)|,
v(s)=\sqrt{\tfrac{1}{2}\,\mathrm{KL}\!\big(\pi_{cg}(\cdot\mid s)\,\|\,\pi(\cdot\mid s)\big)}$, we can have
$$\begin{aligned}
\big|J(\pi_{cg})-J(\pi)\big|
&\le \frac{2}{1-\gamma}\,
\sqrt{\mathbb{E}_{s\sim d^{\pi_{cg}}}\!\big[\sup_a |A^\pi(s,a)|^2\big]}\;
\sqrt{\mathbb{E}_{s\sim d^{\pi_{cg}}}\!\Big[\tfrac{1}{2}\,\mathrm{KL}\!\big(\pi_{cg}(\cdot\mid s)\,\|\,\pi(\cdot\mid s)\big)\Big]}\\
&= \frac{2}{1-\gamma}\,
\sqrt{\mathbb{E}_{s\sim d^{\pi_{cg}}}\!\big[\sup_a |A^\pi(s,a)|^2\big]}\;
\sqrt{\tfrac{1}{2}\,\mathbb{E}_{s\sim d^{\pi_{cg}}}\!\big[\mathrm{KL}\!\big(\pi_{cg}(\cdot\mid s)\,\|\,\pi(\cdot\mid s)\big)\big]}
\end{aligned}$$

Under the diffusion formulation, the causility-guided policy $\pi_{cg}$ differs from the base policy \(\pi\) through additional time-dependent drift terms in the reverse SDE. Girsanov's theorem (or the likelihood ratio for continuous-time diffusions) gives, under suitable regularity and when the two processes share the same diffusion coefficient $g(t)$, that the KL divergence between the conditional action distributions (or between path measures conditioned on $s$) can be written as the expected time-integral of the squared drift difference normalized by the diffusion magnitude. Formally,

$$\mathrm{KL}\!\big(\pi_{cg}(\cdot\mid s)\,\|\,\pi(\cdot\mid s)\big)
= \mathbb{E}_{\pi_{cg}}\!\left[\int_0^T
\Big\|\frac{\gamma_t \nabla_{a_t}\log p_{\varphi}(\cdot\mid\operatorname{\mathbf{do}}(a_t))
+\beta_t \nabla_{a_t}\log p_{\omega}(\cdot\mid\operatorname{\mathbf{do}}(a_t))}{g(t)}\Big\|^2 dt
\right]$$ where the expectation on the right is taken over the policy induced by $\pi_{cg}$. 

Finally, we can conclude that $$\begin{aligned}
&\big|J(\pi_{cg})-J(\pi)\big|\\
&\le \frac{1}{1-\gamma}\,
\sqrt{\mathbb{E}_{s\sim d^{\pi_{cg}}}\!\big[\sup_a |A^\pi(s,a)|^2\big]}
\cdot
\sqrt{\tfrac{1}{2}\,\mathbb{E}_{\pi_{cg}}\!\left[\int_0^T
\Big\|\frac{\gamma_t \nabla_{a_t}\log p_{\varphi}(\cdot\mid\operatorname{\mathbf{do}}(a_t))
+\beta_t \nabla_{a_t}\log p_{\omega}(\cdot\mid\operatorname{\mathbf{do}}(a_t))}{g(t)}\Big\|^2 dt\right]}
\end{aligned}$$
\end{proof}

Theorem~\ref{theor: PDLcau} shows that the performance difference of \textbf{CGDC} is controlled by the KL divergence between the causality-guided policy and the original diffusion policy.
Importantly, the KL divergence grows \textbf{linearly} with the magnitude of the added drift terms. Since the performance difference scales with the \textbf{square root} of the KL divergence, the resulting performance degradation increases only \textbf{sub-linearly} with respect to the drift magnitude, which means even if the causality guidance is not perfectly accurate, it will not lead to catastrophic performance collapse of the diffusion policy.

\subsection{Gradient of Guidance}

Recently, several works have investigated diffusion models with guidance from a more rigorous theoretical perspective \citep{hu2023self, guo2024gradient, jiao2025towards}. In particular, \citet{guo2024gradient} showed that gradient of well-designed guidance can serve as an estimator of the policy gradient. For example, for a policy is in linear Gaussian form $\mathcal{N}\left(g^{\top} a_0, \Sigma\right)$ with diffusion action $a_0$ (denoised from $a_t$ at time $t$) with parameter $g$ and covariance $\Sigma$, a good guidance  $\nabla_{a_t}\mathrm{G}\left(a_t, t\right)$ can be: $\nabla_{a_t}\mathrm{G}\left(a_t, t\right) \propto \nabla_{a_t}g^{\top} \mathbb{E}\left[a_0 \mid a_t\right]$.

Motivated by this theoretical insight, we establish the following proposition: the causal guidance term can serve as an estimator of the gradient of the causal value function, which enables direct optimization of causal rewards in a diffusion policy without explicitly modeling the value function.

\begin{prop}\label{pro:pro2}

Under Assumptions~\ref{assump:lipschitz}, let the gradient of causal value function be defined as $Q_{\text {causal }}\left(s_t, a_t\right) \triangleq \mathbb{E}\left[r_t \mid s_t, \operatorname{\mathbf{do}}(a_t)\right]
$. Then the causal guidance term $g(a) \triangleq \nabla_a \log p\left(s_{t+1}, r_t \mid s_t, \operatorname{\mathbf{do}}(a_t)\right)$ used in the diffusion policy provides an unbiased estimator of the policy gradient as

$$\nabla_a Q_{\text {causal }}\left(s_t, a_t\right) \propto \nabla_a \log p\left(s_{t+1}, r_t \mid s_t, \operatorname{\mathbf{do}}(a_t)\right)$$
\end{prop}

\begin{proof}
 Since $Q_{\text {causal }}\left(s_t, a_t\right) \triangleq \mathbb{E}\left[r_t \mid s_t, \operatorname{\mathbf{do}}(a_t)\right]$, we can have:  

$$\begin{aligned} \nabla_{a_t} Q_{\text {causal }}\left(s_t, a_t\right)&=\nabla_{a_t} \int r_t p\left(s_{t+1}, r_t \mid s_t, \operatorname{\mathbf{do}}(a_t)\right) d s_{t+1} d r_t\\ & =\int r_t \nabla_{a_t} p\left(s_{t+1}, r_t \mid s_t, \operatorname{\mathbf{do}}(a_t)\right) d s_{t+1} d r_t \\ & =\int r_t p\left(s_{t+1}, r_t \mid s_t, \operatorname{\mathbf{do}}(a_t)\right) \nabla_{a_t} \log p\left(s_{t+1}, r_t \mid s_t, \operatorname{\mathbf{do}}(a_t)\right) d s_{t+1} d r_t\\ & =\mathbb{E}\left[r_t \nabla_{a_t} \log p\left(s_{t+1}, r_t \mid s_t, \operatorname{\mathbf{do}}(a_t)\right)\right]\end{aligned}$$

Then, for one Monte Carlo sample of the reward $r_t$, we can have: $\nabla_{a_t} Q_{\text {causal }}\left(s_t, a_t\right) \approx r_t \nabla_{a_t} \log p\left(s_{t+1}, r_t \mid s_t, \operatorname{\mathbf{do}}(a_t)\right)
$. With the normalization of $r_t$, we then have $$\nabla_{a_t} Q_{\text {causal }}\left(s_t, a_t\right) \propto \nabla_{a_t} \log p\left(s_{t+1}, r_t \mid s_t, \operatorname{\mathbf{do}}(a_t)\right)$$

\end{proof}
Proposition~\ref{pro:pro2} demonstrates that our causal guidance is not merely a heuristic for directing reward, but is aligned with stochastic gradient ascent on the expected reward. Moreover, it is agnostic to the specific policy formulation, ensuring scalability to different diffusion policy scenarios.

Beyond reward optimization in diffusion policies, our causal guidance is general and can be seamlessly incorporated into score-matching–based diffusion approaches.
Similar to \textit{Lemma 3} in \cite{guo2024gradient}, we extend the result to causal guidance in the score function $\mathbf{s}_t(x_t)=\nabla_{x_t} \log p_t(x_t)$.

\newtheorem{lemma}{Lemma}

\begin{lemma}
\label{lem:causal_diffusion}
Consider a standard forward diffusion on actions $\mathbf{s}_t\left(a_t\right)=\nabla_{a_t} \log p_t\left(a_t\right)$, which is approximated by the score network $\mathbf{s}_\theta\left(a_t^k, t\right)$. In the backward (denoising) process, if we replace the score with the causal-guided score $\mathbf{s}_\theta(a^k_t,t) + G_{\mathrm{causal}}(a^k_t,t)$ where $G_{\mathrm{causal}}(a^k_t,t) := \nabla_{a^k_t} \log p(s_{t+1},r_t \mid s_t, \operatorname{\mathbf{do}}(a^k_t))$, 
then, in the limit as $t \rightarrow \infty$, the terminal action distribution satisfies 
\begin{align*} 
a_t^{k} \;&\stackrel{d}{=} \mathcal{N}\Big(
\bar{\mu} + \frac{y - M \bar{\mu}}{\Sigma_y + M \bar{\Sigma} M^\top} \bar{\Sigma} M^\top,\;
\bar{\Sigma} - \bar{\Sigma} M^\top (\Sigma_y + M \bar{\Sigma} M^\top)^{-1} M \bar{\Sigma}
\Big)\\&=p\left(a_t^k \mid y\right)\end{align*}
 where
\begin{itemize}
    \item $\bar{\mu}$: the mean of the prior distribution of the action $a^k_t$, i.e., $\bar{\mu} = \mathbb{E}[a^k_t]$ under the forward diffusion.
    \item $\bar{\Sigma}$: the covariance of the prior action distribution, $\bar{\Sigma} = \mathrm{Cov}[a^k_t]$.
    \item where $$y = 
        \begin{bmatrix} s_{t+1} \\ r_t \end{bmatrix}, \quad
        M = \begin{bmatrix} F_a \\ B_a \end{bmatrix}, \quad
        \Sigma_y = 
        \begin{bmatrix} \Sigma_\varphi & 0 \\ 0 & \Sigma_\omega \end{bmatrix}$$

\end{itemize}

Note: Following the causal dynamical model in Eq.~\ref{eq:linearmap}, we rewrite it in terms of linear operators $F$, $B$ for convenience as follows: $$
\begin{aligned}
s_{t+1} & =F_s s_t+F_a a_t+\xi_s, \quad \xi_s \sim \mathcal{N}\left(0, \Sigma_{\varphi}\right) \\
r_t & =B_s s_{t+1}+B_a a_t+\xi_r, \quad \xi_r \sim \mathcal{N}\left(0, \Sigma_\omega\right)
\end{aligned}
$$ where

$$
M=\left[\begin{array}{l}
F_a \\
B_a
\end{array}\right]
$$

\begin{proof}
We can first have the score with causal guidance as: 
\begin{align*}
&\mathbf{s}_\theta(a^k_t,t) + G_{\mathrm{causal}}(a^k_t,t)\\&=\nabla_{a_t^k} \log p\left(a_t^k\right)+\nabla_{a_t^k} \log p\left(s_{t+1}, r_t \mid s_t, \operatorname{\mathbf{do}}\left(a_t^k\right)\right) \\&= \nabla_{a_t^k} \log p\left(a_t^k \mid s_{t+1}, r_t, s_t\right)\end{align*} which implies that $p\left(a_t^k \mid s_{t+1}, r_t, s_t\right) \propto p\left(s_{t+1}, r_t \mid s_t, \mathbf{d o}\left(a_t^k\right)\right) p\left(a_t^k\right)$. 

Then,as the causal dynamical model in Eq.~\ref{eq:linearmap}, under the intervention $\mathbf{d o}\left(a_t^k\right)$, the next state and reward satisfy

$$
\begin{aligned}
s_{t+1} & =F_s s_t+F_a a_t^k+\xi_s, & \xi_s & \sim \mathcal{N}\left(0, \Sigma_{\varphi}\right) \\
r_t & =B_s s_{t+1}+B_a a_t^k+\xi_r, & \xi_r & \sim \mathcal{N}\left(0, \Sigma_\omega\right)
\end{aligned}
$$

Conditioned on $s_t$, both $s_{t+1}$ and $r_t$ depend linearly on $a_t^k$.
Stacking them yields the linear-Gaussian observation model

$$
y:=\left[\begin{array}{c}
s_{t+1} \\
r_t
\end{array}\right]=M a_t^k+\eta, \quad \eta \sim \mathcal{N}\left(0, \Sigma_y\right)
$$ where
$$
M=\left[\begin{array}{l}
F_a \\
B_a
\end{array}\right], \quad \Sigma_y=\left[\begin{array}{cc}
\Sigma_{\varphi} & 0 \\
0 & \Sigma_\omega
\end{array}\right]
$$ which implies that 

$$\left[\begin{array}{c}a_t^k \\ y\end{array}\right] \sim \mathcal{N}\left(\left[\begin{array}{c}\bar{\mu} \\ M \bar{\mu}\end{array}\right],\left[\begin{array}{cc}\bar{\Sigma} & \bar{\Sigma} M^{\top} \\ M \bar{\Sigma} & \Sigma_y+M^2 \bar{\Sigma} M^{\top}\end{array}\right]\right)$$

Meanwhile, we have concluded that $p\left(a_t^k \mid s_{t+1}, r_t, s_t\right) \propto p\left(s_{t+1}, r_t \mid s_t, \operatorname{\mathbf{do}}\left(a_t^k\right)\right) p\left(a_t^k\right)$ before. Therefore, then $t \rightarrow \infty$, we can have $ a_t^k \stackrel{d}{=}p\left(a_t^k \mid y\right)$. 

\end{proof}

\end{lemma}
Lemma~\ref{lem:causal_diffusion} shows that our proposed causal guidance preserves exact posterior sampling over actions under intervention conditioning within \textbf{score-matching diffusion policies}. 
This implies that the guidance does not violate the fundamental principles of DDIM or other score-based diffusion methods, even in the presence of interventions. 
Consequently, the proposed causal guidance is scalable and fully compatible with arbitrary score-based diffusion policies.


\section{Experiments}\label{sec:exper}

We evaluate our method on a diverse set of tasks, including Gym MuJoCo benchmarks \citep{todorov2012mujoco} and D4RL benchmarks \citep{fu2020d4rl}, to assess the performance of CausalGDP.

\paragraph{Baselines.} We select a diverse set of representative reinforcement learning approaches as baselines in our experiments, covering model-based, value-based, policy-based, and sequence modeling paradigms in offline RL. Specifically, model-based and conservative methods include MoRel \citep{kidambi2020morel}, SAC \citep{haarnoja2018soft}, CQL \citep{kumar2020conservative}, and BCQ \cite{fujimoto2019off}, which aim to mitigate distributional shift through explicit or implicit uncertainty modeling. Value-based and policy-learning approaches such as IQL \citep{kostrikov2021offline}, QSM \citep{psenka2023learning}, AWAC \citep{nair2020awac}, BEAR \cite{kumar2019stabilizing}, BRAC \citep{wu2019behavior}, REM \citep{agarwal2020optimistic}, AWR \citep{peng2019advantage}, TD3 \citep{fujimoto2018addressing}, TD3+BC \citep{fujimoto2021minimalist}, and Onestep RL \citep{brandfonbrener2021offline} emphasize stable policy improvement under offline constraints. In addition, we include a broad class of sequence modeling and diffusion-based methods, including DT \cite{chen2021decision}, TT \citep{janner2021offline}, Diffuser \citep{janner2022planning}, DD \citep{ajay2022conditional}, Diffusion-QL \cite{wang2022diffusion}, QT \citep{hu2024q}, QDT \citep{yamagata2023q}, DIPO \citep{yang2023diffusion}, QVPO \citep{ding2024diffusion}, RVDT \citep{bai2025rebalancing}, and TCD \citep{hu2026instructed}, which model decision-making as conditional sequence generation. Together, these baselines provide a comprehensive comparison across prevailing policy of RL methodologies.  For all baseline methods, we adopt the best-performing results as reported in their orignial publications.

\paragraph{Tasks Setting.} For Gym MuJoCo tasks, we evaluate our method on HalfCheetah, Hopper, Walker2d, and Humanoid. Rather than relying on the fixed offline datasets provided by D4RL (v2), we collect experience directly from the updated Gym MuJoCo-v4 environments through online interaction. The resulting transitions are stored in a replay buffer and subsequently used for training in an off-policy manner. This setting allows us to evaluate the effectiveness of causal-guided diffusion policies under a more flexible and realistic data regime. For D4RL, we consider tasks from the Maze2D, Kitchen, AntMaze, and Adroit domains. We set the number of training episodes to 2000 for HalfCheetah, Hopper, and Walker2d, and 3000 for Humanoid in the Gym MuJoCo tasks. For all D4RL tasks, the number of training episodes is set to 1000. These diverse tasks cover a broad range of real-world robotic and automation scenarios. Specifically, Gym MuJoCo benchmarks focus on humanoid and articulated physical control, Maze2D and AntMaze evaluate route planning and long-horizon navigation, while Kitchen and Pen-human represent dexterous manipulation tasks relevant to robotic arms in manufacturing settings. Evaluating across this diverse set of tasks enables a comprehensive assessment of the effectiveness and scalability of our CausalGDP framework.

 We report the experimental results in Table~\ref{tab:main_results}. As shown in the table, our CausalGPD consistently outperforms the baseline methods across the majority of tasks, with particularly strong gains over diffusion-based policy and planning approaches. These results indicate that causality-guided learning consistently strengthens diffusion-based policies across diverse tasks, as action selection and policy updates are informed by causal relationships among states, actions, and rewards during training. 
 
 \paragraph{Maze2D and AntMaze.} For example, on the Maze2D and AntMaze benchmarks, CausalGDP achieves strong performance on the large-scale tasks, where navigation routes are significantly more complex. This indicates that causality-guided learning is particularly effective for long-horizon route planning, with potential implications for real-world applications such as autonomous driving and robotic delivery. 
 
 \paragraph{Adroid and Kitchen.} On the Adroit manipulation tasks, CausalGDP consistently outperforms other diffusion-based policy approaches, including DD and Diffusion-QL.
We observe that Kitchen is substantially more challenging to learn, as it involves long-horizon human-interaction behaviors. As a result, several baseline methods (e.g., BEAR and BARC) fail to make progress and achieve near-zero scores in this environment.
In contrast, CausalGDP is able to attain strong performance even on this difficult task, demonstrating its expressiveness and scalability.

 \paragraph{Gym MuJoCo.} It is worth noting that D4RL only provides fixed offline datasets for Gym MuJoCo-v2. In contrast, our Gym tasks are generated by simulating offline training data from the Gym MuJoCo-v4 environments, where the collected trajectories are comparable in distribution to the ``medium'' or ``medium-replay'' datasets in D4RL. We additionally present representative examples of average return curves of Halfcheetah-v4  and Humanoid-v4 in Fig.~\ref{fig:causal_mujoco} to illustrate the convergence advantages of our approach. In particular, causality-guided learning enables the RL policy to reach high-reward regions more rapidly while maintaining stable training dynamics. A theoretical analysis supporting these observations is provided in Sec.~\ref{sec:theory}. In particular, Humanoid-v4 is widely regarded as one of the most challenging Gym benchmarks due to its high-dimensional state space and complex action space. Nevertheless, our CausalGDP achieves strong performance on this task as well.

 \paragraph{Ablation Study.} We investigate how the accuracy and source of causal information affect the performance of CausalGDP. Specifically, we consider two causal discovery methods—NOTEARS \cite{zheng2018dags} and LDirectLiNGAM \cite{shimizu2011directlingam}—as well as a noise-injected causal setting in which perturbations are added during causal relationship learning to intentionally produce biased and inaccurate causal structures. These three settings allow us to construct three distinct types of causal information used by CausalGDP and systematically analyze their impact. 

We report the results of three variants of CausalGDP in Table~\ref{tab:abliastudy}.
The results indicate that different causal discovery methods are all suitable for learning causal information within our framework, as their performances are largely comparable. This demonstrates the robustness of CausalGDP with respect to the choice of causal discovery method. Importantly, although inaccurate causal information may lead to a slight degradation in performance, the diffusion policy remains stable and consistently achieves competitive results comparable to other baselines.
 \begin{table}[H]
\centering
\captionsetup{font=footnotesize}
\resizebox{\textwidth}{!}{
\begin{tabular}{@{}lccccc@{}}
\toprule
\textbf{Task} & \textbf{CausalGDP-N} & \textbf{CausalGDP-DL} & \textbf{CausalGDP-WN}    \\ \midrule
\textbf{maze2d-umaze-v1} & 165.4 $\pm$6.2 & 163.1 $\pm$ 4.8& 137.5 $\pm$ 5.0 \\
\textbf{maze2d-medium-v1} & 175.3 $\pm$ 4.3& 177.0 $\pm$ 2.6 & 158.3 $\pm$ 6.5  \\
\textbf{maze2d-large-v1} & 266.2 $\pm$ 4.6 & 263.7 $\pm$ 4.0& 228.3 $\pm$ 8.5
 \\
\bottomrule
\end{tabular}
}
\caption{
Ablation study of CausalGDP with different sources and accuracies of causal information.
CausalGDP-N uses causal graphs discovered by NOTEARS,
CausalGDP-DL employs causal structures learned via LDirectLiNGAM,
and CausalGDP-WN injects noise during causal relationship learning to produce intentionally biased and inaccurate causal information.}\label{tab:abliastudy}
\end{table}

\begin{figure}[H]
    \centering
    \begin{subfigure}{0.48\textwidth}
        \centering
        \includegraphics[width=\linewidth]{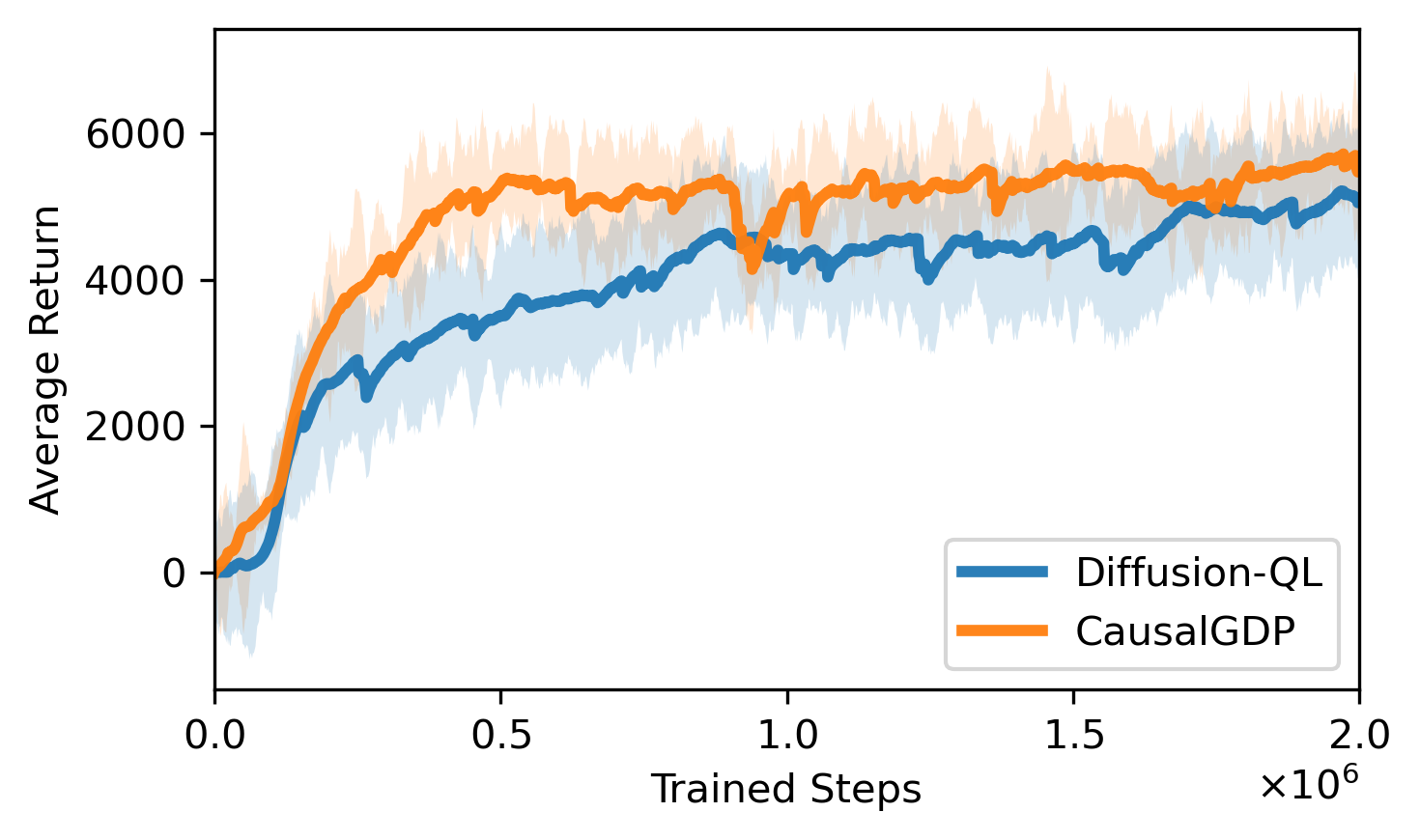}
        \caption{Halfcheetah}
        \label{fig:causal_half}
    \end{subfigure}
    \hfill
    \begin{subfigure}{0.48\textwidth}
        \centering
        \includegraphics[width=\linewidth]{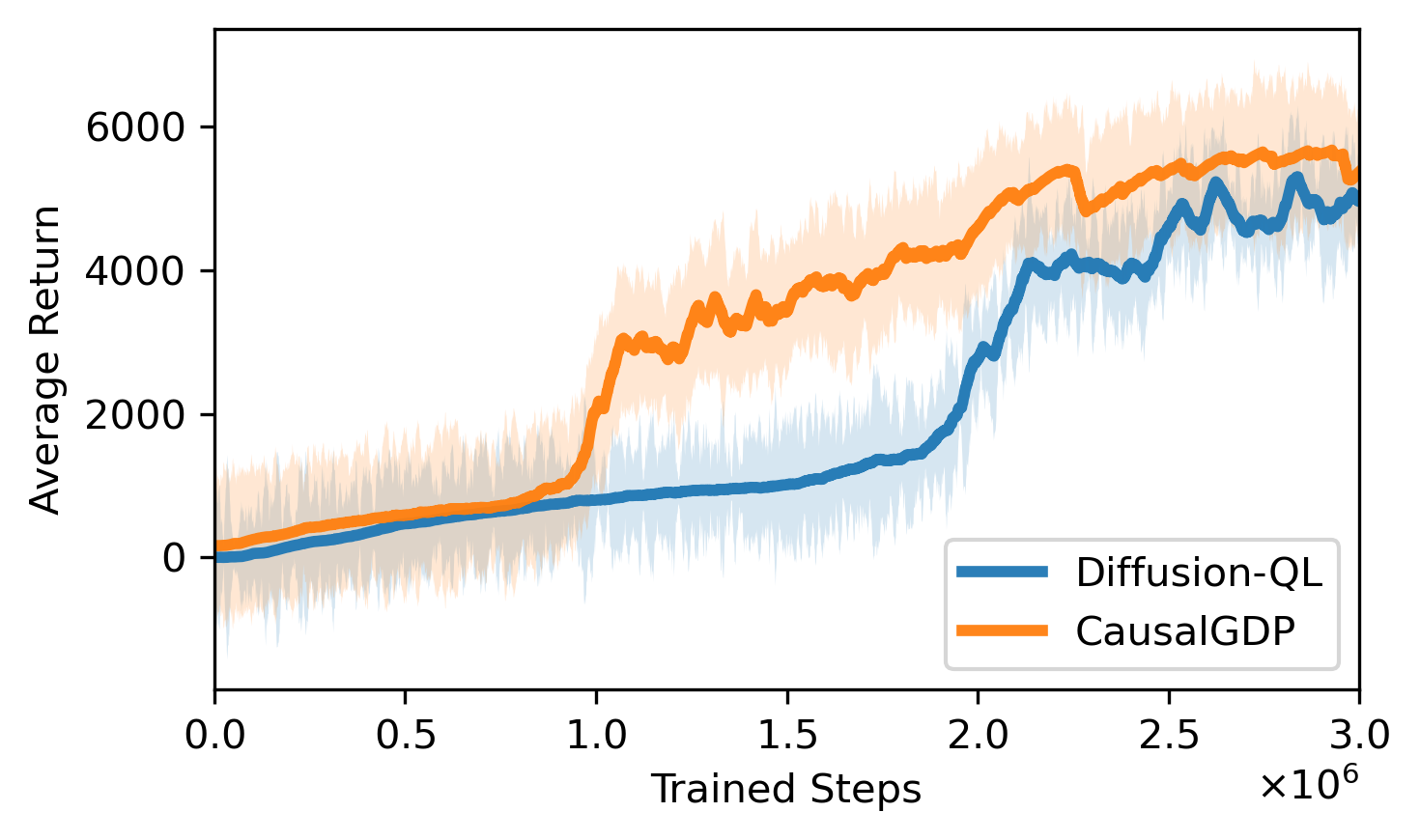}
        \caption{Humanoid}
        \label{fig:causal_human}
    \end{subfigure}
    \caption{Performance comparison of CausalGDP on Gym MuJoCo tasks. Learning curves of average episodic return on HalfCheetah-v4 and Humanoid-v4, comparing CausalGDP with Diffusion-QL over training steps.}
    \label{fig:causal_mujoco}
\end{figure}

\begin{table}[H]
\centering
\footnotesize
\caption{\small\textbf{Performance comparison on Gym MuJoCo and D4RL benchmarks.}
This table reports average normalized returns across multiple continuous-control tasks from Gym MuJoCo and D4RL, including Maze2D, AntMaze, Adroit manipulation, and Kitchen tasks. For each benchmark, results are averaged over all constituent tasks, with mean and standard deviation computed over multiple random seeds when available. Higher scores indicate better performance. CausalGDP consistently achieves competitive or superior performance across diverse benchmarks, particularly on Maze2D, AntMaze, and Humanoid tasks.}
\setlength{\tabcolsep}{3.5pt}
\renewcommand{\arraystretch}{1.1}
\resizebox{0.98\textwidth}{!}{%
\begin{tabular}{l c c c c c c c c c}
\toprule
\midrule
\multicolumn{10}{l}{\textbf{Maze2D}} \\
\midrule
\textbf{Task} 
& \textbf{QT} 
& \textbf{DT} 
& \textbf{DD} 
& \textbf{TCD} 
& \textbf{QDT} 
& \textbf{CQL} 
& \textbf{RVDT} 
& \textbf{Diffusion-QL} 
& \textbf{CausalGDP} \\
\midrule
maze2d-umaze-v1         
& 105.4 
& 27.3 
& 17.16 
& 134.2 
& 40.0 
& 94.7 
& 145.1 
& 159.3 
& \textbf{165.4 $\pm$ 6.2} \\
maze2d-medium-v1        
& 172.0 
& 32.1 
& -3.1 
& 28.2 
& 13.3 
& 41.8 
& \textbf{183.5} 
& 174.2 
& 176.3 $\pm$ 4.3\\
maze2d-large-v1         
& 240.1 
& 18.1 
& -14.2 
& 7.7 
& 31.0 
& 49.6 
& 254.3 
& 252.9 
& \textbf{266.2 $\pm$ 4.6} \\
maze2d-umaze-dense-v1   
& 103.1 
& -6.8 
& 83.2 
& 30.0 
& 58.6 
& 72.7 
& 99.5 
& \textbf{161.6}
& 136.2 $\pm$ 5.0 \\
maze2d-medium-dense-v1  
& 111.9
& 31.5 
& 78.2 
& 41.4 
& 42.3 
& 70.9 
& 126.9 
& 184.3 
& \textbf{197.2 $\pm$ 3.7} \\
maze2d-large-dense-v1   
& 177.2
& 45.3 
& 23.0 
& 75.5 
& 62.2 
& 90.9 
& 197.9
& 181.4
& \textbf{198.6 $\pm$ 3.5} \\
\midrule
Average                 
& 151.6 
& 24.6 
& 30.9 
& 52.8 
& 41.2
& 70.1 
& 167.9 
& 185.7 
& \textbf{190.0} \\

\midrule\midrule

\textbf{Gym} \\
\midrule
\textbf{Task} 
& \textbf{Diffuser} 
& \textbf{MoRel} 
& \textbf{Onestep RL} 
& \textbf{TD3+BC} 
& \textbf{DT} 
& \textbf{CQL} 
& \textbf{IQL} 
& \textbf{Diffusion-QL} 
& \textbf{CausalGDP} \\
\midrule 

Halfcheetah & 5152 & 4937 & 5023 & 5164 & 4996 & 5168 & 5176 & 5109 & \textbf{5217.4 $\pm$ 70.2} \\
Hopper      & 2223 & \textbf{3627} & 2227 & 2261 & 2578 & 2233 & 2525 & 3020.8 & 3106.6 $\pm$ 212.3\\
Walker2d    & 3036 & 2490 & 2475 & 3591 & 3331 & 3660 & 3460 & \textbf{3581.0} & 3423.6 $\pm$ 343.1 \\
\midrule
Average & 3470.3 & 3684.7 & 3241.7 & 3672.0 & 3635.0 & 3687.0 & 3720.3 &3903.6 &\textbf{3915.9} \\

\midrule
\midrule
\textbf{Task} 
& \textbf{SAC} 
& \textbf{QSM} 
& \textbf{TD3} 
& \textbf{TD3+BC} 
& \textbf{DIPO} 
& \textbf{ QVPO} 
& \textbf{IQL} 
& \textbf{Diffusion-QL} 
& \textbf{CausalGDP} \\
\midrule 
Humanoid    & 4996.29    &  4793.1    & 5035.36    & 4688    & 4945.6    & 5306    & 4660    & 5339.7 & \textbf{5443.1 $\pm$ 71.1} \\
\midrule\midrule

\multicolumn{10}{l}{\textbf{AntMaze}} \\
\midrule
\textbf{Task} 
& \textbf{BCQ} 
& \textbf{BEAR} 
& \textbf{Onestep RL} 
& \textbf{TD3+BC} 
& \textbf{DT} 
& \textbf{CQL} 
& \textbf{IQL} 
& \textbf{Diffusion-QL} 
& \textbf{CausalGDP} \\
\midrule
antmaze-umaze-v0          & 78.9 & 73.0 & 64.3 & 78.6 & 59.2 & 74.0 & 87.5 & 93.4 &  \textbf{95 $\pm$ 2.2}\\
antmaze-umaze-diverse-v0  & 55.0 & 61.0 & 60.7 & 71.4 & 53.0 & \textbf{84.0} & 62.2 & 66.2 &  68.5 $\pm$ 6.8\\
antmaze-medium-play-v0    & 0.0  & 0.0  & 0.3  & 10.6 & 0.0  & 61.2 & 71.2 & \textbf{76.6} & \textbf{76.5 $\pm$ 9.0} \\
antmaze-medium-diverse-v0 & 0.0  & 8.0  & 0.0  & 3.0  & 0.0  & 53.7 & 70.0 & 78.6  & \textbf{83.2 $\pm$ 8.1} \\
antmaze-large-play-v0     & 6.7  & 0.0  & 0.0  & 0.2  & 0.0  & 15.8 & 39.6 & \textbf{46.4} & 43.1 $\pm$ 7.5 \\
antmaze-large-diverse-v0  & 2.2  & 0.0  & 0.0  & 0.0  & 0.0  & 14.9 & 47.5 & 56.6 & \textbf{58.3 $\pm$ 7.1}\\
\midrule
Average & 23.8 & 23.7 & 20.9 & 27.3 & 18.7 & 50.6 & 63.0 & 69.6 & \textbf{70.8} \\
\midrule\midrule

\multicolumn{10}{l}{\textbf{Adroit}} \\
\midrule
\textbf{Task} 
& \textbf{BCQ} 
& \textbf{BEAR} 
& \textbf{DD} 
& \textbf{BRAC-v} 
& \textbf{REM} 
& \textbf{CQL} 
& \textbf{IQL} 
& \textbf{Diffusion-QL} 
&  \textbf{CausalGDP} \\
\midrule
pen-human-v1  & 68.9 & -1.0 & 66.7 & 0.6  & 5.4  & 35.2 & 71.5 & 72.8  & \textbf{73.9 $\pm$ 8.0} \\
pen-cloned-v1 & 44.0 & 26.5 & 42.8 & -2.5 & -1.0 & 27.2 & 37.3 & 57.3  & \textbf{58.0 $\pm$ 15.4} \\
\midrule
Average & 56.5 & 12.8 & 4.9 & -1.0 & 2.2 & 31.2 & 54.4 & 65.1 & \textbf{66.0}\\
\midrule\midrule

\multicolumn{10}{l}{\textbf{Kitchen}} \\
\midrule
\textbf{Task} 
& \textbf{BCQ} 
& \textbf{BEAR} 
& \textbf{BRAC-p} 
& \textbf{BRAC-v} 
& \textbf{AWR} 
& \textbf{CQL} 
& \textbf{IQL} 
& \textbf{Diffusion-QL} 
&  \textbf{CausalGDP}\\
\midrule
kitchen-complete-v0 & 8.1 & 0.0 & 0.0 & 0.0 & 0.0 & 43.8 & 62.5 & \textbf{84.0}  & 72.2 $\pm$ 4.1 \\
kitchen-partial-v0  & 18.9 & 13.1 & 0.0 & 0.0 & 15.4 & 49.8 & 46.3 & \textbf{60.5} & \textbf{60.5 $\pm$ 2.7}  \\
kitchen-mixed-v0    & 8.1 & 47.2 & 0.0 & 0.0 & 10.6 & 51.0 & 51.0 & 62.6 & \textbf{65 $\pm 5.4$} \\
\midrule
Average & 11.7 & 20.1 & 0.0 & 0.0 & 8.7 & 48.2 & 53.3 & \textbf{69.0} & 65.8\\

\bottomrule
\end{tabular}
}

\label{tab:main_results}
\end{table}

\section{Conclusion}
In this work, we proposed \textbf{CausalGDP}, a causality-guided diffusion policy framework that explicitly integrates causal reasoning into diffusion-based reinforcement learning. By learning and continuously updating causal dynamics models, CausalGDP guides action generation toward components that genuinely influence future states and rewards. Empirical results show that CausalGDP is robust to different causal discovery methods and achieves competitive performance in complex, high-dimensional control tasks. We believe that this form of causal guidance can be extended to broader real-world scenarios, which we leave for future research.

\newpage
\appendix
\section{Appendix}

\begin{table}[H]
\centering
\footnotesize
\setlength{\tabcolsep}{6pt}
\renewcommand{\arraystretch}{1.15}
\begin{tabular}{l c c c c}
\toprule
\textbf{Task} & \textbf{Learning Rate} & $\boldsymbol{\eta}$ &\textbf{Hidden Units} & \textbf{N of Layers}\\
\midrule
\multicolumn{3}{l}{\textbf{Maze Tasks}} \\
\midrule
maze2d-umaze-v1     & $3 \times 10^{-4}$ & 3.0 &128 &3 \\
maze2d-medium-v1              & $3 \times 10^{-4}$ & 5.0&128&3 \\
maze2d-large-v1        & $3 \times 10^{-4}$ & 5.0 &128 &3\\
maze2d-umaze-dense-v1     & $3 \times 10^{-4}$ & 3.0&128 &3 \\
maze2d-medium-dense-v1       & $3 \times 10^{-4}$ & 5.0&128&3 \\
maze2d-large-dense-v1     & $3 \times 10^{-4}$ & 5.0&128&3 \\
\midrule
\multicolumn{3}{l}{\textbf{Gym}} \\
\midrule
Halfcheetah     & $3 \times 10^{-4}$ & 5.0 &256 &4 \\
Hopper              & $1 \times 10^{-4}$ & 2.0 &128 &4 \\
Walker2d        & $1 \times 10^{-4}$ & 2.0 &256 &4 \\
Humanoid     & $3 \times 10^{-4}$ & 5.0 &256 &4 \\

\midrule
\multicolumn{3}{l}{\textbf{AntMaze Tasks}} \\
\midrule
antmaze-medium-diverse-v0     & $3 \times 10^{-4}$ & 3.0 &128 &3 \\
antmaze-umaze-v0              & $3 \times 10^{-4}$ & 0.5  &128 &3\\
antmaze-large-play-v0         & $3 \times 10^{-4}$ & 4.5 &128 &3 \\
antmaze-umaze-diverse-v0      & $3 \times 10^{-4}$ & 2.0  &128 &3\\
antmaze-medium-play-v0        & $1 \times 10^{-3}$ & 2.0 &128 &3\\
antmaze-large-diverse-v0      & $3 \times 10^{-4}$ & 3.5 &128 &3\\
\midrule
\multicolumn{3}{l}{\textbf{Adroit Tasks}} \\
\midrule
pen-cloned-v1                 & $3 \times 10^{-5}$ & 0.1 &256 &3 \\
pen-human-v1                  & $3 \times 10^{-5}$ & 0.15 &256 &3\\
\midrule
\multicolumn{3}{l}{\textbf{Kitchen Tasks}} \\
\midrule
kitchen-mixed-v0              & $3 \times 10^{-4}$ & 0.005 &128 &4 \\
kitchen-complete-v0           & $3 \times 10^{-4}$ & 0.005 &128 &4\\
kitchen-partial-v0            & $3 \times 10^{-4}$ & 0.005 &128 &4 \\
\bottomrule
\end{tabular}
\caption{Hyperparameter settings used across different task domains. \textbf{Hidden Units} and \textbf{N of Layers} specify the architecture of the multi-layer perceptron (MLP) used to parameterize the \emph{causal dynamical model}, which estimates the causal relationships between states, actions, and subsequent states or rewards.}
\label{tab:hyperparams_structured}
\end{table}

\bibliography{mybib}

\end{document}